\documentclass[10pt,conference,letterpaper]{IEEEtran}
\usepackage{times,amsmath,epsfig}

\usepackage{booktabs} 
\usepackage[T1]{fontenc}
\usepackage{lineno}
\usepackage{amsmath,mathtools,amsthm}
\usepackage{amssymb}
\usepackage{amstext}
\usepackage{tikz}
\usetikzlibrary{calc}
\usepackage{subfig}
\usepackage{url}
\usepackage{multirow}
\usepackage[linesnumbered,ruled]{algorithm2e}
\usepackage{paralist}
\usepackage{enumitem}
\usepackage{rotating}
\usetikzlibrary{arrows,decorations.pathmorphing,backgrounds,positioning,fit,petri}
\usepgflibrary{decorations.markings} 
\usepgflibrary[decorations.markings] 
\usetikzlibrary{decorations.markings} 
\usetikzlibrary[decorations.markings] 
\usetikzlibrary{calc,intersections,through,backgrounds}

\modulolinenumbers[5]

\newcommand{\subseqs}{\mathcal{S}}

\newcommand{\featurespace}{\mathbb{F}}

\newcommand{\subs}{\sqsubseteq}

\newcommand{\ranking}{\mathbf{r}}
\newcommand{\rankings}{\mathcal{R}}

\newcommand{\google}{\mathbb{G}}
\newcommand{\bing}{\mathbb{B}}
\newcommand{\adjacency}{\mathbf{A}}
\newcommand{\bA}{\mathbf{A}}
\newcommand{\bL}{\mathbf{L}}
\newcommand{\bM}{\mathbf{M}}
\newcommand{\bI}{\mathbf{I}}
\newcommand{\ties}{\mathcal{T}}

\newtheorem{theorem}{Theorem}

\newtheorem{corollary}{Corollary}

\newtheorem{example}{Example}

\DeclareMathOperator{\tr}{tr}
\DeclareMathOperator{\nnz}{nnz}

\begin{document}
\title{ Consensus measure of rankings }

\author{%
	{Zhiwei Lin{\small $~^{\#1}$},  Yi Li{\small $~^{*2}$},Xiaolian Guo{\small $~^{\#3}$} }%
	\vspace{1.6mm}\\
	\fontsize{10}{10}\selectfont\itshape
	$^{\#}$\,School of Computing  and Mathematics, Ulster University\\
	Northern Ireland, United Kingdom\\
	\fontsize{9}{9}\selectfont\ttfamily\upshape
	$^{1}$\,z.lin@ulster.ac.uk\\
	$^{3}$\,guo-x3@email.ulster.ac.uk%
	\vspace{1.2mm}\\
	\fontsize{10}{10}\selectfont\rmfamily\itshape
	$^{*}$\,Division of Mathematics, SPMS\\ Nanyang Technological University\\
	 21 Nanyang Link, Singapore\\
	\fontsize{9}{9}\selectfont\ttfamily\upshape
	$^{2}$\,yili@ntu.edu.sg
}

\maketitle

%
%
%
%
%
%
%
%
\begin{abstract}

A ranking is an ordered sequence of   items, in which an item with higher ranking score is more preferred than  the items with lower ranking scores.   In many information systems,  rankings are widely used to represent the preferences over a set of items or candidates. The consensus measure  of rankings  is the problem of how to evaluate  the degree to which the rankings agree.  The consensus measure can be used to evaluate rankings in many information systems,  as quite often  there is not ground truth available for evaluation.

This paper introduces a novel approach for consensus measure of rankings by using graph representation, in which the vertices or nodes are the items and the edges are the relationship of items in the rankings.  Such representation leads to various algorithms for consensus measure in terms of different aspects of rankings, including the number of common patterns, the number of common patterns with fixed length and the length of the longest common patterns.  The proposed measure can be adopted for various types of rankings, such as full rankings, partial rankings and rankings with ties.  This paper  demonstrates how the proposed approaches can be used to evaluate the quality of rank aggregation and the quality of top-$k$ rankings from Google and Bing search engines.

\end{abstract}
 

\section{Introduction}
In many information systems,  rankings are widely used to represent the preferences over a set of items or candidates, ranging from information retrieval, recommender to  decision making systems \cite{Liang:2014:TRA,Ivan:group:decision:2012,Schubert:ranking:outliers:2012,Farnoud:2014,Georgoulas:2017,Choudhury:icde:2017}, in order to improve quality of the services provided by the systems. For example, in search engine, the list of the terms suggested by a search engine after a user's few keystrokes is a typical ranking and such ranking service, widely adopted nowadays, has great impact on user's search experience; it is also recognized that the list of search results is a ranking after a query is issued.

A ranking is an ordered sequence of   items, in which an item with higher ranking score is more preferred than  the items with lower ranking scores.   The consensus   of rankings  is the degree to which the rankings agree according to  certain common patterns. The consensus measure,  can be used in many information systems, in order to uncover how close or related the rankings are.  For example, in the group decision making, a group of experts express their preferences over a set of candidates by using rankings and the measure of the degree of consensus is very useful for reaching  consensus \cite{Ivan:group:decision:2012}. 

{
In many information system with large volume of items, such as search engines, it is hard to clearly define what ground truth is, which make it more difficult to evaluate  and compare the rankings returned from the systems.    
The consensus measure  of rankings, as a tool for understanding how related or close the rankings are, will  help engineers and researchers  to discern  what aspects of a ranking system need to be improved and to detect outliers \cite{Cormode:ICDE:2009,Schubert:2012}.  For a set of rankings $\rankings=\{\ranking_1, \dots, \ranking_n  \}$,  one approach to understanding the degree to which rankings agree  is to use rank correlation or similarity function by pairwise comparison \cite{Schubert:ranking:outliers:2012,Yilmaz:2008:rank:correlation,Kumar:2010:GDR,Carterette:2009:RCD, Kendall:1938:rank:correlation,Kendall:correlation:book,Spearman:1904:measurement,Webber:2010,Vigna:2015:rankings:ties}.  The notable functions include   the  Kendall index  $\tau(\ranking_i,\ranking_j)$ and the Spearman index $\rho(\ranking_i,\ranking_j)$ \cite{Kendall:1938:rank:correlation,Spearman:1904:measurement}, which however  do not have a weighting scheme so that less important items can be penalized. It is common that in information retrieval, the documents (items) at the top of a ranking list are more important than those at the bottom \cite{Fagin:2003}. As such, it makes sense to  reduce the impact from the bottom items with a    weighting scheme. For example, the variation of $\tau$ index, denoted by $\tau_{ap}$,  with  average precision,  is able to   give greater weight to the top items of the ranking lists \cite{Yilmaz:2008:rank:correlation}.   
These methods assume  rankings are conjoint, meaning that items in the rankings are completely overlapped. Undoubtedly, they cannot be used for     partial rankings, in which items may not be mutually overlapped. As a similarity function for two partial  rankings, the RBO (rank-biased overlap) proposes to weight the number of common items according to the depth of rankings \cite{Webber:2010}, and it doesn't take into account the order of items in the rankings. 

When one of the correlation or similarity functions is   used for consensus measure for a set of $n$ rankings in $\rankings$, we can aggregate the pairwise comparison values across all rankings for $\binom{n}{2}$ times. Since the pairwise comparison is based on the degree of commonality in   two rankings, with respect to features or patterns (e.g, the common items, or the concordant pairs against the disconcordant pairs),   the aggregated result is not informative enough to tell the extend to which   the rankings agree  in $\rankings$, according to the study by  Elzinga et al.~\cite{Cees:2011:concordance}. Also, the type of rankings can be full or partial, specially the top-$k$ \cite{Fagin:2003}, the existing measures fail to meet the requirements for handling different types of rankings.  
}

In order to effectively evaluate and compare rankings,   which could be full or partial and especially in which some items need to be weighted, this paper propose a new approach based on graph representation. 
The novelty of this paper lies in that fact the new proposed  consensus measure of rankings does not  need pairwise comparison, which is significantly different from the pairwise approaches using similarity or correlation  functions. The contribution of the paper includes:
\begin{itemize}[leftmargin=1cm] 	
	\item we introduce a directed acyclic graph (DAG) to represent the relationship between items in the rankings so that such representation can be used to induce   efficient algorithms for consensus measure of rankings;
	\item  the proposed representation of DAG   enables us to  approach consensus measure of rankings in terms of different aspects of the   common features or patterns hidden in the rankings, including  $\kappa(\rankings)$ -- the number of common patterns, $\kappa_p(\rankings)$ -- the number of common patterns with a fixed length $p$, and  $\ell(\rankings)$ -- the length of the longest or largest common patterns. 
	\item the proposed representation of DAG is extended to allow the edges in the graph to have weights so that more ``important'' features or patterns are assigned with higher values and the features or patterns with less ``importance'' are penalized. 
	\item we also demonstrate that the consensus measure of rankings with graph representation can  be extended to calculate  consensus measure for  duplicate rankings, for rankings with ties and for  rankings whose top items need to be weighted. 
	\item we show that our approach can be used for different types of rankings, including the full rankings and top-$k$ rankings. 
\end{itemize}  
 The rest of  paper is organized as follows. Section \ref{section:preliminaries} introduces the important notations and concepts used in the paper, followed by a review of related work   in Section \ref{section:related:work}.  Section \ref{section:consensus} presents a directed  graph representation approach for consensus measure. Section \ref{section:experiments} shows how the proposed approaches can be used to evaluate rank aggregation and to compare top-$k$ rankings. The paper is concluded  in Section \ref{section:conclusion}.
  

\section{Preliminaries} \label{section:preliminaries}

This section introduces notations and concepts of graph, ranking sequence, and consensus measure  that will be used in the rest of the paper. 

\subsection{Directed graph}
A {\em directed graph} is a pair $G=(V,E)$, where $V$ is the set of  nodes (or vertices) and $E$ is the set of directed edges. A directed edge $(x,y)$ means that the edge leaves node $x\in V$ and enters node $y\in V$.  An edge $(x,x)$ is called a loop, which leaves node $x$ and returns to itself. Given a graph $G=(V,E)$ with $n=|V|$ nodes,  matrix $\mathbf{A}=\left(A_{ij}\right)_{n\times n }$  is used to denote the {\em  adjacency matrix} of graph $G=(V,E)$,  where $A_{ij}=1$ if there exists edge $(x_i,x_j)\in E$; and $A_{ij}=0$, otherwise. 

The adjacency matrix $\adjacency$ assumes that all the edges have identical weights of 1, and this can be relaxed in the {\em weighted directed graph}. A weighted directed graph  $G=(V,E, W)$ is a directed graph, in which $W$ is a set of weights on the edges and each edge $(x_i,x_j)\in E$ is assigned a non-zero  weight $w(i,j) \in W$.  Then, the adjacency matrix $\adjacency$ for $G=(V,E, W)$ is defined as $A_{ij}=w(i,j)$ if $(x_i,x_j)\in E$ and $A_{ij}=0$, otherwise. 

A  {\em path} from node $x_i$ to $x_j$ is  a sequence of distinct  non-loop edges \[(x_i,x_{k_1}), (x_{k_1},x_{k_2}),\dots, (x_{k_{p}}, x_j)\] connecting   node $x_i$  and  $x_j$.  

\subsection{Ranking sequences}
A ranking $\ranking$ is an ordered sequence $\ranking=(\sigma_{i_1},\sigma_{i_2},\ldots,\sigma_{i_m})$ of $m$ distinct items drawn from a universe $\Sigma=\{\sigma_1,\cdots,\sigma_n \}$, where $m\leq n$ and $\sigma_{i_j}$ is more preferred than $\sigma_{i_k}$ if $i_j<i_k$. The length of $\ranking$ is denoted by $|\ranking|$. For notational simplicity, we shall simply write a ranking as a sequence of $\ranking= \sigma_{i_1}\sigma_{i_2}\cdots \sigma_{i_m}$ in the rest of the paper. 

For a ranking $\ranking=r_{1}\cdots r_{k}$, where $r_{j}\in \Sigma$ for $1\leq j\leq k$, we can  define the embedded patterns with respect to {\em subsequences}. A sequence $\ranking'=r'_{1}\cdots r'_{m}$ is called a {\em subsequence} of $\ranking$, denoted by $\ranking'\sqsubseteq \ranking$, if $\ranking'$ can be  obtained by  deleting $k-m$ items from $\ranking$. We denote by $\ranking'\not\sqsubseteq \ranking$ that $\ranking'$ is not a subsequence of  $ \ranking$.  For example,  $bde \sqsubseteq abcde$, and $bac\not\sqsubseteq abcde$.

A ranking sequence with no items is an {\em empty sequence}. We use $\subseqs(\ranking)$ to denote the set of all possible non-empty subsequences of $\ranking$.   $\subseqs(\ranking)$ can be partitioned into subsets $\subseqs_p(\ranking)$, where $\subseqs_p(\ranking)$ consists of all subsequences of length $p$. For example, if $\ranking=abcde$, then  $\subseqs_3(\ranking)=\{abc,  abd, abe, acd, ace, ade, bcd, bce, bde, cde\}$, in which each subsequence has length $3$.

The degree to which rankings agree lies in the common patterns or features which are embedded in the rankings. For ranking sequences, the subsequences are the patterns or features. Given a set of $N$ rankings $\rankings=\{\ranking_1,\dots,\ranking_N  \}$, consider $\subseqs(\rankings) = \subseqs(\ranking_1)\cap \cdots \cap\subseqs(\ranking_N)$,  each element $x\in\subseqs(\rankings)$ is a {\em common subsequence} of $\ranking_1,\dots,\ranking_m $,   for which we also use the notation $x\sqsubseteq \rankings$.  Similar to   $\subseqs_p(\ranking)$, we   also define    $\subseqs_p(\rankings)$ to denote the subsets of all common subsequences of length $p$. Therefore, it holds that   $\subseqs(\rankings)=\bigcup_{1\leq p\leq l} \subseqs_p(\rankings)$, where $l=\min\{ |\ranking| : \ranking\in \rankings \}$. In a special case, for two rankings $\ranking_i$ and $\ranking_j$, we will write $\subseqs_p(\ranking_i,\ranking_j)$ to denote the set of  $p$-long common subsequences between  $\ranking_i$ and $\ranking_j$.

It is clear that $\subseqs(\rankings)$ accommodates all common features (subsequences), which are subsumed by each ranking $\ranking\in \rankings$. Let $\kappa(\rankings)$ denote the number of all common subsequences of $\rankings$, i.e,    
\begin{equation}\label{eq:def:measure}
\kappa(\rankings)=|\subseqs(\rankings)|.
\end{equation}
The more common features $\subseqs(\rankings)$ has or the bigger $\kappa(\rankings)$ is, the higher degree of consensus $\rankings$ has.  We   also define 
\begin{equation}
\kappa_p(\rankings)=|\subseqs_p(\rankings)|
\end{equation}
 in order to measure the consensus in $\rankings$ with respect to the number of subsequences of a given length $p$.  The length of the longest common subsequences of rankings in $\rankings$ is denoted by  $\ell(\rankings)$ or simply $\ell$. Then, $\ell=\max \{|z| : z \in \subseqs(\rankings)  \}$.

Therefore, we have the following properties:
\begin{itemize}
	\item For a set   with only one  ranking $\rankings=\{\ranking\}$, where $n=|\ranking|$, $\kappa(\rankings)=2^n-1$;
	\item For a set of two   rankings $\rankings=\{\ranking_x, \ranking_y\}$, where $m=|\ranking_x|$ and $n=|\ranking_y|$, we have 
		$0\leq \kappa(\rankings)\leq 2^{\min\{m,n \}}-1;$
	\item For two sets of rankings $\rankings_x$ and $\rankings_y$, if $\rankings_x\subseteq\rankings_y$, then
	\begin{eqnarray}\label{eq:kappa:property3}
		\kappa(\rankings_y)\leq \kappa(\rankings_x)
	\end{eqnarray}
\end{itemize}

\subsection{Consensus measure of rankings in feature spaces}
For a set of $n$ rankings $\rankings$, we can form a set of features $\featurespace=\bigcup_{\ranking\in \rankings} \subseqs(\ranking)$. 
Let $m=|\featurespace|$ and $\featurespace = \{y_1,\dots,y_m\}$. Each ranking $\ranking$ can be represented by a feature vector with a mapping function $\phi : \rankings \rightarrow \{0,1\}^m $:
\[
\phi(\ranking)= \left(f_\ranking(y_1), \dots, f_\ranking(y_m)  \right),
\]
where 
\begin{equation}\label{eq:feature:mapping}
	f_\ranking(y_k)=
	\begin{cases}
		1 &  y_k \subs \ranking \\
		0 &  y_k \not\subs \ranking
	\end{cases}
\end{equation}
It is clear that $\kappa(\rankings)$, defined in Equation \eqref{eq:def:measure} can be rewritten using the   inner product  on $n$-inner product spaces \cite{Gunawan2002,MANA:n:inner:product} as 
\begin{equation}\label{eq:feature:space}
	\kappa(\rankings) = \langle \phi(\ranking_1), \dots, \phi(\ranking_n)  \rangle 
	= \sum_{k=1}^{m} \prod_{\ranking\in \rankings}   f_{\ranking} (y_k) 
\end{equation}
With the generalized inner product, we find that  the $\kappa(\rankings)$ is a kernel function \cite{ShaweTaylor_book} when $|\rankings|=2$. The rewritten $\kappa(\rankings)$ relies on the definition of $f_\ranking(y_k)$ as defined in  Equation \eqref{eq:feature:mapping}, whose co-domain is $\{0, 1\}$. It is computationally expensive to enumerate all the features and to form $\featurespace$. In this section, we will transform  relationship between items to a graph so that efficient algorithms can be found without  enumerating features explicitly, which is similar to the kernel trick for kernel functions \cite{ShaweTaylor_book}.

\section{Related work}\label{section:related:work}

A ranking can be full or partial ranking, depending on the number of items from $\Sigma$ being ranked.
A ranking $\ranking$ is a {\em full} ranking if  $|\Sigma|=|\ranking|$. A  ranking $\ranking$ is called {\em partial }  ranking if the items in $\ranking$ forms a subset of $\Sigma$. A top-$k$ ranking is a sub-ranking of full ranking but only with the top-$k$ items. 
{\em Rankings with ties} occur when some items share an identical ranking score, which happens very often in the decision making or voting process \cite{Vigna:2015:rankings:ties}. For example, in a ranking $\ranking=\{a\}\{bc\}\{d\}$, both items $b$ and $c$ are assigned with an identical ranking score.

Evaluation or comparison of rankings is an important tasks in many ranking related systems, including decision making, information retrieval, voting and recommender \cite{Liang:2014:TRA,Ivan:group:decision:2012}. One approach to evaluating rankings is to use rank correlation between two rankings.  The widely used Kendall $\tau$ index \cite{Kendall:1938:rank:correlation,Kendall:correlation:book} is a measure of rank correlation between two rankings  $\ranking_i$ and $\ranking_j$ over $n$ items  by taking into account 2-long common subsequences between them, which  can be formulated   as 
\[
\tau(\ranking_i,\ranking_j)=\frac{|\subseqs_2(\ranking_i,\ranking_j)|-|\subseqs_2(\overleftarrow{\ranking_i},\ranking_j)|}{\binom{n}{2}}
\]
where $\overleftarrow{\ranking_i}$ is a reverse ranking of ${\ranking_i}$. In rank aggregation,  one could also use the Kendall distance $d_\tau(\ranking_i,\ranking_j)$ -- a variation of the Kendall $\tau$:
\[
d_\tau(\ranking_i,\ranking_j)=\frac{|\subseqs_2(\overleftarrow{\ranking_i},\ranking_j)|}{\binom{n}{2}}
\]

The Spearman $\rho$ index is another measure of rank correlation that does not utilize the 2-long common subsequences but  instead takes into account  of each item positions in $\ranking_i$ and $\ranking_j$ \cite{Spearman:1904:measurement}.  It is defined as follows
\[
\rho(\ranking_i,\ranking_j)=1-\frac{6\sum_{\sigma\in \Sigma} \left( \eta_i(\sigma)-\eta_j(\sigma) \right)^2 }{n(n^2-1)}
\]
where $n=|\Sigma|$. The Spearman footrule distance $d_\rho$ is an $L_1$ distance, which is a variation of the Spearman $\rho$:
\[
d_\rho(\ranking_i,\ranking_j)=\frac{\sum_{\sigma\in \Sigma} | \eta_i(\sigma)-\eta_j(\sigma) | }{n(n^2-1)}
\]

Compared with the Spearman $\rho$,  the Kendall $\tau$ ignores the use of items positions, which are in many cases very important factors, e.g,  for the top-$k$ rankings. Again, the Spearman $\rho$ can not be used for sensitivity detection and analysis, as studied in \cite{Kendall:correlation:book}.

Both Kendall and Spearman can only be used for full rankings. They cannot be used for partial or top-$k$ rankings.  Even in full rankings, both of them lack of weighting schemes and are not flexible enough for rankings whose items at the top are more important than the items at the bottom \cite{Fagin:2003}.  Therefore, it is necessary to    reduce the impact from the bottom items with a  down-weighting scheme for those bottom items. For example, the variation of $\tau$ index, denoted by $\tau_{ap}$,  with  average precision,  is able to   give greater weight to the top items of the ranking lists \cite{Yilmaz:2008:rank:correlation}.  Shieh also developed a weighted metric $\tau_w$ based on the Kendall $\tau$  by adding weighting factors to the 2-long subsequences \cite{SHIEH199817}.  For full rankings with ties, $\tau_{t}$  was proposed based on the   Kendall index \cite{Vigna:2015:rankings:ties}. One extension  $\rho_w$ to the Spearman index by Iman et al. was to assign higher weights to the items at the top \cite{Iman:1987}. 

The above methods assume  that  rankings are full rankings, meaning that items in the rankings are completely overlapped.   Therefore, they cannot be used for partial rankings. In information retrieval, it is more interesting to compare the the rankings based on their  top-$k$  items. Fagin et al. proposed two measures $\tau_k$ and $\rho_k$ by adapting both Kendall $\tau$ and Spearman $\rho$ for top-$k$ rankings  \cite{Fagin:2003}.  As a similarity function for two partial  rankings, the RBO (rank-biased overlap) proposes to weight the number of common items according to the depth of rankings \cite{Webber:2010}, but it does not take into account the order of items in the rankings.

\begin{table}[!t]
	\caption{A summary of the popular indices for various types of rankings with comparison to  our approach of  $\kappa_p(\rankings)$ and  $\kappa(\rankings)$.} \label{table:summary}
	\begin{center}
		\begin{tabular}{c|c|c|c|c}
			\toprule
	  			     &Full & Partial & Weighted & Ties \\\hline
			 $\tau$  & x  && \\\hline
			 $ \rho$ & x&&\\\hline 
		  $\tau_{ap}$& x& &x \\\hline
			 $\tau_w$&  x& &x  \\\hline
			 $\rho_w$&  x& &x 	\\\hline
		     $\rho_k$& & x& 	\\	\hline
		     RBO	 &	 & x& x\\ \hline
		     $\tau_{t}$&   x& & &x \\\hline
		    $\kappa_p(\rankings)$&   x& x&x &x \\\hline
		    $\kappa(\rankings)$&   x& x&x &x \\
		   	\hline
		\end{tabular}
		
	\end{center}	
\end{table}

These functions are pairwise comparison and they can be transferred into consensus measure     for a set $\rankings$ of $n$ rankings    by  aggregating the pairwise distance values across all rankings. For example, one can use $\sum_{i=1}^{n} \sum_{j=1, i\neq j}^{n} \tau(\ranking_i,\ranking_j)$ if the Kendall index is preferred.  However,  this aggregated result is not informative enough to tell the extend to which the rankings agree  in $\rankings$, according to the study by  Elzinga et al.~\cite{Cees:2011:concordance}.

We summarize the popularly used   indices in Table \ref{table:summary} and we show that our approach of  $\kappa_p(\rankings)$ and  $\kappa(\rankings)$ is more flexible for various types of rankings, which will be demonstrated in the next section. Also, those existing indices shown in Table \ref{table:summary}  cannot be used for sensitivity detection in the consensus measure, while our approach has the ability to discern how the rankings come to agree by varying the parameters to the gaps and positions of items, as pointed out in Section \ref{section:sensitivity} and as verified in Section \ref{section:experiments}.

\section{Graph representation for consensus measure of rankings}\label{section:consensus}

This section will introduce a graph approach to consensus measure of rankings by calculating $\kappa(\rankings)$ and $\kappa_p(\rankings)$.

\subsection{A motivating example}\label{section:motivating:example}

Consider a set of rankings $\rankings=\{ \ranking_1=abcdef, \ranking_2=bdcefa,  \ranking_3=bcdeghijkf, \ranking_4=badefc\}$. Without loss of generality, we randomly pick $\ranking_1\in\rankings$ (note that $|\ranking_1|=6$) and form a lower triangle matrix $\mathbf{A}=\left(A_{ij}\right)_{6\times 6}$ of size $6\times 6$, where for $i\geq  j$,  $A_{ij}=1$ if the $i^\text{th}$ item and the $j^\text{th}$ item of $\ranking_1$ both occur in the same order in all rankings in $\rankings$, and $A_{ij}=0$ otherwise. Then we obtain  matrix $\mathbf{A}$:
\begin{equation}\label{eq:example:I}
\mathbf{A}_{6\times 6}
=\quad \bordermatrix{~ & a &b &c &d &e &f\cr
	a & 0 & 0 & 0 & 0 & 0 & 0\cr
	b & 0 & 1 & 0 & 0 & 0 & 0\cr
	c & 0 & 1 & 1 & 0 & 0 & 0\cr
	d & 0 & 1 & 0 & 1 & 0 & 0\cr
	e & 0 & 1 & 0 & 1 & 1 & 0\cr
	f & 0 & 1 & 0 & 1 & 1 & 1\cr}
\end{equation}

With matrix $\adjacency$, we can induce  a weighted directed   graph   $G=(V,E_\ell\cup E_e)$ on the diagonal elements of $\mathbf{A}$, where $V = \{A_{11},\dots, A_{66}\}$ is the set of vertices, $E_\ell$ is the set of loops and $E_e$ is the set of non-loop edges. Later, we may use $V = \{a,b, c, d, e, f \}$ interchangeably without confusion as each $A_{ii}$ stands for an item. Hereinafter in this paper, we shall distinguish loops and non-loops edges, and abuse the notation and simply call the latter edges. 

The edges are drawn according to the following: for $1\leq i, j\leq |\ranking|$,  an edge from $A_{ii}$ to $A_{jj}$ is added if the following conditions all hold: (1) $i<j$;  (2) $A_{ii}=A_{jj}=1$; (3) $A_{ji}\neq 0$. We also add dashed loops on diagonal elements of value $1$. Figure \ref{figure:example} shows the    weighted directed   graph for the matrix $\mathbf{A}$, in which there are seven directed (solid) edges, i.e,
\begin{align*}
  E_e &= \big\{(A_{22},A_{33}), (A_{22},A_{44}), (A_{22},A_{55}), (A_{22},A_{66}), \\
    & \qquad (A_{44},A_{55}),  (A_{44},A_{66}),(A_{55},A_{66}) \big\} 
\end{align*}or 
simply 
$E_e=\big\{(b,c),  (b,d), (b,e), (b,f), (d,e), (d,f), (e,f) \big\}. $
Those edges are the 2-long common subsequences: $bc, bd, be, bf,  de, df, ef$, and all of them occur  in $\ranking_1$, $\ranking_2$, $\ranking_3$ and $\ranking_4$. As such, $\kappa_2(\rankings)=|E_e|=7$.
	Similarly, paths\footnote{Recall that our definition of path excludes loop edges.} of length $3$ corresponds to common subsequences of length $3$. We find that $\kappa_3(\rankings)=4$ with common subsequences being $bde$, $bdf$, $bef$ and $def$. Next, $\kappa_4(\rankings)=1$ with the common subsequence being $bdef$. There is no longer common subsequences since the length of the longest path in $G$ is $4$.
 
%

In Figure \ref{figure:example}, the five dashed loops mean five singletons, i.e., $b$, $c$, $d$, $e$, $f$. As a result, $\kappa_1(\rankings)=5$. Therefore, we obtain
 $\kappa(\rankings)=\kappa_1(\rankings)+\kappa_2(\rankings)+\kappa_3(\rankings)+\kappa_4(\rankings)=5+7+4+1=17$.

\begin{figure}[t]
	\begin{center}
		{
			\begin{tikzpicture}
			\draw (-1.1,6.1) -- (4.8,6.1) --(4.8,0.4)--(-1.1,0.4)--(-1.1,6.1);
			
			\draw[xstep=1.1cm,ystep=1cm,gray, dotted,very thin] (-1,0.7) grid (4.4,5.7);
			\foreach \i [count=\xi] in {b,c,d,e,f}{
				\draw ({(\xi-1)*1.1},6.3) node {$\i$};					
				\draw ({(\xi-1)*1.1},5.8) node {0};
			}
			\draw (-0.9,6.3) node {$a$};
			\draw (-0.9,5.8) node {0};
			\draw (-1.5,5.8) node {$a$};
			\foreach \i [count=\xi] in {b,c,d,e, f}{
				\draw (-1.5, 6-\xi) node {$\i$};
				
			}
			\draw (-0.9,6.8-2) node[anchor=south] {0};					
			\draw (-0.9,6.8-3) node[anchor=south] {0};					
			\draw (-0.9,6.8-4) node[anchor=south] {0};					
			\draw (-0.9,6.8-5) node[anchor=south] {0};					
			\draw (-0.9,6.8-6) node[anchor=south] {0};					
			
			\draw (0,5) node (0005) {\textbf{1}};
			\draw (0,4) node (0004) {\textbf{1}};
			\draw (0,3) node (0003) {\textbf{1}};
			\draw (0,2) node (0002) {\textbf{1}};
			\draw (0,1) node (0001) {\textbf{1}};
			
			\draw (1.1,4) node (0104) {\textbf{1}};
			\draw (1.1,3) node (0103) {\textbf{0}};
			\draw (1.1,2) node (0102) {\textbf{0}};
			\draw (1.1,1) node (0101) {\textbf{0}};

			\draw (2.2,3) node (0203) {\textbf{1}};
			\draw (2.2,2) node (0202) {\textbf{1}};
			\draw (2.2,1) node (0201) {\textbf{1}};
			
			\draw (3.3,2) node (0302) {\textbf{1}};
			\draw (3.3,1) node (0301) {\textbf{1}};

			\draw (4.4,1) node (0401) {\textbf{1}};
			
			\draw [->, red]  (0005) to [out=300,in=150,looseness=0] (0104);
			\draw [->, red]  (0005) to [out=330,in=100,looseness=1] (0203);
			\draw [->, red]  (0005) to [out=360,in=80,looseness=1] (0401);	
			\draw [->, red]  (0005) to [out=340,in=100,looseness=1] (0302);
			\draw [->, red]  (0203) to [out=300,in=145,looseness=0] (0302);
			\draw [->, red]  (0302) to [out=300,in=140,looseness=0] (0401);
			\draw [->, red]  (0203) to [out=340,in=100,looseness=1] (0401);
			
			\draw [->,densely dashed, red]  (0104) to [out=170,in=250,looseness=7] (0104);
			\draw [->,densely dashed, red]  (0005) to [out=170,in=250,looseness=7] (0005);
			\draw [->,densely dashed, red]  (0203) to [out=170,in=250,looseness=7] (0203);
			\draw [->, densely dashed, red]  (0302) to [out=170,in=250,looseness=7] (0302);
			\draw [->, densely dashed,  red]  (0401) to [out=170,in=250,looseness=7]  (0401);
			
			\end{tikzpicture}
		}
	\end{center}
	\caption{ The  weighted directed graph of the matrix $\mathbf{A}$ (in Equation \eqref{eq:example:I}) with directed edges $(A_{ii}, A_{jj})$,   from $A_{ii}$   to $A_{jj}$, where $i<j$, if $A_{ji}\neq 0$, $A_{ii}\neq 0$, and  $A_{jj}\neq 0$.}	\label{figure:example}
\end{figure}
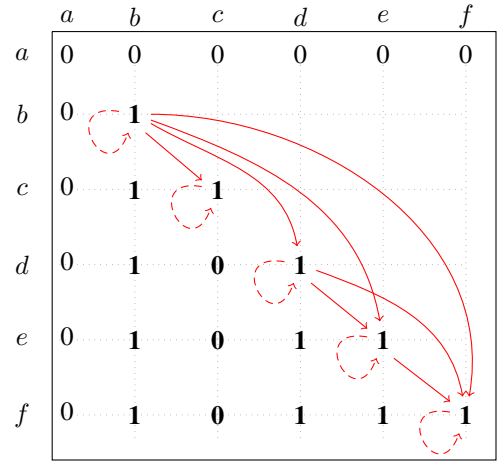

This process of finding patterns of various lengths  with graph representation not only allows us to calculate $\kappa_p(\rankings)$, but also makes it easy for us to calculate the number of all  common patterns $\kappa(\rankings)$ and the length of the longest common subsequences $\ell(\rankings)$.

\subsection{Consensus measure by graph representation }\label{section:new:framework}

\begin{figure}[t]
	\begin{center}
		{
			\begin{tikzpicture}
			\draw (-1.1,6.1) -- (4.8,6.1) --(4.8,0.3)--(-1.1,0.3)--(-1.1,6.1);
			
			\draw[xstep=1.1cm,ystep=1cm,gray, dotted,very thin] (-1,0.6) grid (4.4,5.7);
			\foreach \i [count=\xi] in {b,c,d,e,f}{
				\draw ({(\xi-1)*1.1},6.3) node {$\i$};					
				\draw ({(\xi-1)*1.1},5.8) node {0};
			}
			\draw (-0.9,6.3) node {$a$};
			\draw (-0.9,5.8) node {0};
			\draw (-1.5,5.8) node {$a$};
			\foreach \i [count=\xi] in {b,c,d,e, f}{
				\draw (-1.5, 6-\xi) node {$\i$};
				
			}
			\draw (-0.9,6.8-2) node[anchor=south] {0};					
			\draw (-0.9,6.8-3) node[anchor=south] {0};					
			\draw (-0.9,6.8-4) node[anchor=south] {0};					
			\draw (-0.9,6.8-5) node[anchor=south] {0};					
			\draw (-0.9,6.8-6) node[anchor=south] {0};					
			
			\draw (0,5) node (0005) {$\boldsymbol{\theta(b)}$};
			\draw (0,4) node (0004) {$\boldsymbol{\psi(b,c)}$};
			\draw (0,3) node (0003) {$\boldsymbol{\psi(b,d)}$};
			\draw (0,2) node (0002) {$\boldsymbol{\psi(b,e)}$};
			\draw (0,1) node (0001) {$\boldsymbol{\psi(b,f)}$};
			
			\draw (1.1,4) node (0104) {$\boldsymbol{\theta(c)}$};
			\draw (1.1,3) node (0103) {\textbf{0}};
			\draw (1.1,2) node (0102) {\textbf{0}};
			\draw (1.1,1) node (0101) {\textbf{0}};

			\draw (2.2,3) node (0203) {$\boldsymbol{\theta(d)}$};
			\draw (2.2,2) node (0202) {$\boldsymbol{\psi(d,e)}$};
			\draw (2.2,1) node (0201) {$\boldsymbol{\psi(d,f)}$};
			
			\draw (3.3,2) node (0302) {$\boldsymbol{\theta(e)}$};
			\draw (3.3,1) node (0301) {$\boldsymbol{\psi(e,f)}$};
			
			\draw (4.4,1) node (0401) {$\boldsymbol{\theta(f)}$};
			
			\draw [->, red]  (0005) to [out=300,in=150,looseness=0] (0104);
			\draw [->, red]  (0005) to [out=330,in=100,looseness=1] (0203);
			\draw [->, red]  (0005) to [out=360,in=80,looseness=1] (0401);	
			\draw [->, red]  (0005) to [out=340,in=100,looseness=1] (0302);
			\draw [->, red]  (0203) to [out=300,in=145,looseness=0] (0302);
			\draw [->, red]  (0302) to [out=300,in=140,looseness=0] (0401);
			\draw [->, red]  (0203) to [out=340,in=100,looseness=1] (0401);
			
			\draw [->,densely dashed, red]  (0104) to [out=200,in=250,looseness=4] (0104);
			\draw [->,densely dashed, red]  (0005) to [out=200,in=250,looseness=4] (0005);
			\draw [->,densely dashed, red]  (0203) to [out=200,in=250,looseness=4] (0203);
			\draw [->, densely dashed, red]  (0302) to [out=200,in=250,looseness=4] (0302);
			\draw [->, densely dashed,  red]  (0401) to [out=200,in=250,looseness=4]  (0401);
			
			
			\end{tikzpicture}
		}
	\end{center}
	\caption{ The   weighted directed graph generalized from the matrix $\mathbf{A}$ (in Equation \eqref{eq:example:I}) with  weights on edges $\psi(\sigma_i, \sigma_j)$ and weights on loops $\theta(\sigma)$.}	\label{figure:example:weighted}
\end{figure}
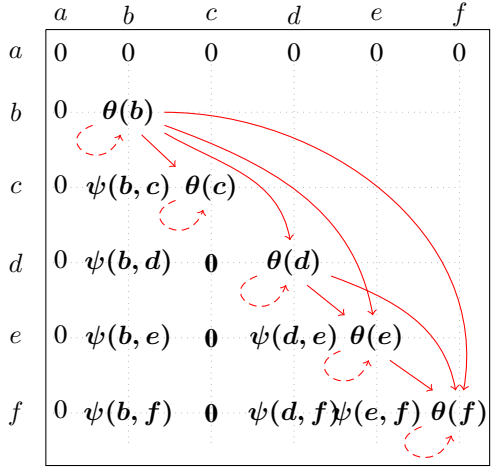

The above example shown in Fig. \ref{figure:example} presents an approach with graph representation for consensus measure of rankings when $f_\ranking(y_k)\in \{0, 1\}$.  This section will extends the graph representation to the consensus measure of rankings by calculating $\kappa(\rankings)$ and $\kappa_p(\rankings)$, when  $f_\ranking(y_k)\in [0, 1]$.

In Equation \eqref{eq:feature:mapping},  the  definition of  $f_\ranking(y_k)\in \{0, 1\}$ assumes that   features in $\featurespace$ are equally assigned  with a weight of 1. However, this is not true in many cases when  some features or items in $\featurespace$ is more important than the others \cite{Yilmaz:2008:rank:correlation,Webber:2010}. The  definition of  $f_\ranking(y_k)\in \{0, 1\}$ is not flexible enough to  differentiate   the importance of the features.  As such, we shall extend it to $f_\ranking(y_k)\in [0,1]$ if $y_k \subs \ranking$, and $f_\ranking(y_k)=0$ if $y_k \not\subs \ranking$ so that ``important'' features will receive higher values of $f_\ranking(y_k)$ while features with less importance will be ``penalized'' with  lower  $f_\ranking(y_k)$. Therefore, we rewrite  Equation \eqref{eq:feature:space} as 
\begin{align} 
\kappa(\rankings) & = \sum_{k=1}^{m} \prod_{\ranking\in \rankings}   f_{\ranking} (y_k) \label{eq:new:feature:space}
\end{align} 
for  $f_{\ranking} (y_k) \in [0,1]$. 

In the DAG shown in Figure \ref{figure:example}, we assume that the weights on the edges   equal to 1, which does not reflect the nature of how each subsequence is embedded in the original rankings. 	Consider the four rankings in  $\rankings=\{ \ranking_1=abcdef, \ranking_2=bdcefa,  \ranking_3=bcdeghijkf, \ranking_4=badefc\}$, item $f$   occurs at different positions in the rankings, which is shown in the following table:
\begin{center}
	\begin{tabular}{cccc}\hline
	$\ranking_1$  &	$\ranking_2$  & $\ranking_3$ & $\ranking_4$ \\
	6 & 5 & 10 & 5\\
	\hline
 \end{tabular}
\end{center}
The position for $f$ in $\ranking_3$ is 10, which deviates    from the positions of $f$ in  the other rankings substantially. In order to incorporate  those   factors which may affect the degree of consensus, we relax the assumption that the weights are identical to 1 and generalize the induced weighted DAG by introducing two functions $\theta(\sigma)$ and $\psi(\sigma_i, \sigma_j)$.  Figure \ref{figure:example:weighted} shows the new DAG, where each edge is associated with weight $\psi(\sigma_i,\sigma_j)$ and each loop is assigned with $\theta(\sigma)$, where   $\psi(\sigma_i,\sigma_j)\in (0,1]$ and  $\theta(\sigma)\in (0,1]$.  We will illustrate how the two functions reflect those factors in the following sections, and how they are related to $f_\ranking(y_k)\in [0,1]$ for Equation \eqref{eq:new:feature:space}.

For simpler presentation of our algorithm, we introduce the (left-continuous) Heaviside function 
\begin{equation} \label{eq:indicator}
H(x)=
\begin{cases}
1 & x> 0;  \\
0 & \text{otherwise}.
\end{cases}
\end{equation}
%

Now we   present the following  theorem for measuring the   consensus of rankings.
\begin{theorem}\label{theorem:framework}
		Given a set $\rankings=\{\ranking_1, \cdots, \ranking_N \}$  of $N$ rankings over a universe $\Sigma$, where each ranking $\ranking_k=r_{k_1}\cdots r_{k_m}$ is naturally associated with a map  $\eta_k: \Sigma\rightarrow \{0, 1, \dots, |\Sigma|\}$ defined as
		\begin{equation}\label{eq:eta}
		\eta_k(\sigma)=
		\begin{cases}
		0, & \sigma\not\sqsubseteq \ranking_k;\\
		j, & \sigma = r_{k_j}.
		\end{cases}
		\end{equation}
		
Let $\ranking_x = r_{x_1}r_{x_2}\cdots r_{x_n}$ be an arbitrary ranking from $\rankings$, $n=|\ranking_x|$, and ${\mathbf{A}}=\left({A}_{ij}\right)_{n\times n}$ be an adjacency  matrix of a graph, where ${A}_{ij}=$
	{\small 
		\begin{equation} \label{eq:matrix:K0:theorem}
	\begin{cases}
	0, & i<j;\\
	\theta(r_{x_i})\prod_{k=1}^{N} H\!\left(\eta_k(r_{x_i})\right),   &i= j; \\
	\psi(r_{x_i}, r_{x_j}) \prod_{k=1}^{N} H\!\left(\eta_k(r_{x_i})\!-\!\eta_k(r_{x_j})\right)H({A}_{ii})H({A}_{jj}), &i> j. \\
	\end{cases}
	\end{equation}}
	and  let ${\mathbf{L}} = \left({L}_{ij}\right)_{n\times n} $ be strictly lower triangle of ${\mathbf{A}}$,   and $\mathbf{z}=(1,\dots, 1)^T$ be a vector of all ones. 
	Then, 
	\begin{equation} \label{eq:kappa:p:new}
	{\kappa}_p(\rankings)=
	\begin{cases}
	\tr(\bA), & p=1;\\
	\mathbf{z}^T {\bL}^{p-1}\mathbf{z}	,   & p>1.	 
	\end{cases}
	\end{equation}
\end{theorem}
\begin{proof}

Note that $\mathbf{z}^T \bM\mathbf{z}$ gives the sum of all entries in a matrix $\bM$. By definition of $
\bA$ we know that $\bA_{ii}>0$ if $r_{x_i}\in \subseqs$ and $\bA_{ii}=0$ otherwise. It follows that $\kappa_1(\bA)$ is the number of $1$s on the diagonal line, which equals to,  noting that all other entries on the diagonal lines are $0$s, the sum of diagonal entries of $\bA$, or $\tr(\bA)$. 

For $p\geq 2$, it is a classical inductive argument that $(\bL^{p-1})_{ij} = N_p(i,j)$ when $i > j$, where $N_p(i,j)$ is the number of common sequences of length $p$ which begin with $r_{x_i}$ and end with $r_{x_j}$. The advertised result then follows from the fact that all rankings have distinct items and thus the common subsequences also have distinct items.
%
\end{proof}

Theorem \ref{theorem:framework} shows that   the individual items are weighted by $\theta(\sigma)$ and the edges between any two items by   $\psi(\sigma_i,\sigma_j)$, reflecting the strength of the relationship between two items $\sigma_i$ and $\sigma_j$.

\subsubsection{\textbf{$\theta(\sigma)$ -- weighted by standard deviation of item's positions}}

The position of item $\sigma$ in a ranking is an indication of the strength of  being preferred. To show the importance of the position of $\sigma$, we define $\mu(\sigma)$, the average of the  positions of $\sigma$ in $\ranking_k$, as follows.  
\begin{equation}
\mu(\sigma)=
\begin{cases}
-\infty & \sigma\not\sqsubseteq \rankings \\
\frac{1}{N}\sum_{k=1}^{N}\eta_k(\sigma)   & \sigma \sqsubseteq \rankings
\end{cases}
\end{equation}

If an item is placed in a small range of positions throughout all rankings, it is assumed that this item is preferred consistently at the same level by all rankings. On the other hand, if an item has a low position $\eta_i(\sigma)$ in one ranking $\ranking_i$ while has a high  position  $\eta_j(\sigma)$ in another ranking $\ranking_j$, the big difference between the positions  $|\eta_j(\sigma)-\eta_i(\sigma)|$ indicates the inconsistency of the preferences over this item. To take into account the differences of item's positions  in consensus measure,   we  define 
\begin{equation} \label{equation:theta}
\theta(\sigma)=\gamma^d,
\end{equation}
where
$d= \frac{1}{N}\sum_{k=1}^{N} |\eta_k(\sigma)- \mu(\sigma)| $, in order  to  weigh the item using the positions $\eta_k(\sigma)$. 
In fact, when feature $y_k$ is a singleton (i.e.,  $y_k=\sigma$), it is clear that  \[\theta(\sigma)=\prod_{\ranking\in \rankings} f_\ranking(\sigma),\] where $f_\ranking(\sigma)=\frac{1}{N}|\eta_k(\sigma)- \mu(\sigma)|$ for Equation \eqref{eq:new:feature:space}. 

\subsubsection{\textbf{$\psi(\sigma_i,\sigma_j)$ -- weighted by gaps}}
The gap between items has been used for pairwise  kernel functions or sensitivity detection \cite{Thompson:gap,ShaweTaylor_book}. Now we extend this to the set-wise consensus measure of $\kappa_p(\rankings)$. 

The  weighted DAG in Figure \ref{figure:example:weighted} shows that the edges $(b,f)$ and $(b,c)$, i.e., subsequences $bf$ and $bc$ are quite different in terms of the distance between $b$ and $f$, and between $b$ and $c$, in $\ranking_1$. There are no items between $b$ and $c$, however $b$ and $f$ are separated by three other items $c$, $d$ and $e$, which means that $c$ is much more preferred than $f$.  Therefore, for each $\ranking_k$ and every 2-long subsequence $\sigma_i\sigma_j$, we define the gap $\varpi_k(\sigma_i,\sigma_j)=\eta_k(\sigma_j)-\eta_k(\sigma_i)$, which indicates  
how much $\sigma_i$ is more preferred than its successor $\sigma_j$ in the 2-long subsequence $\sigma_i\sigma_j$ with respect to the original ranking sequence $\ranking_k$. The following table shows the gaps for $bf$ and $bc$ with respect to the example rankings.

\begin{center}
	\begin{tabular}{l|llll}
		\hline
		~&  $\ranking_1$ & $\ranking_2$ &$\ranking_3$ & $\ranking_4$	\\
		\hline \hline
		$\varpi_k(b,c)$	&  1 & 2 & 1 & 5   \\
		$\varpi_k(b,f)$	&  4 & 4 & 9 &  4  \\
		\hline
		
	\end{tabular}
	
\end{center}

Clearly, the accumulated gaps for $b$ and $f$ is much bigger than that for $b$ and $c$. This suggests that $f$ is less likely to be preferred over $c$. To take into account this likelihood, we define  
\begin{equation}\label{eq:psi}
	\psi(\sigma_i,\sigma_j)=\lambda^g,
\end{equation}
where $0<\lambda\leq 1$ and 	 $g=\frac{1}{N}\sum_{k=1}^{N} |\varpi_k(\sigma_i,\sigma_j)| $, so that any subsequence $\sigma_i\sigma_j$ with bigger average gaps  will be ``penalized''. 

Now we relate $\psi(\sigma_i,\sigma_j)$ to $f_\ranking(y_k)$ for Equation \eqref{eq:new:feature:space}. For $p>1$, a $p$-long subsequence  $y_k=(\sigma_{k_1},\sigma_{k_2},\dots, \sigma_{k_{p}}) \in \featurespace$ is represented by  a $(p-1)$-long path $(\sigma_{k_1},\sigma_{k_2}),  \dots, (\sigma_{k_{p-1}},\sigma_{k_{p}})$ in the graph. As each edge $(\sigma_{k_i},\sigma_{k_j})$ has weight $\psi(\sigma_{k_i},\sigma_{k_j})$, defining  \[f_\ranking(y_k)=\prod_{i=1}^{p-1} \psi(\sigma_{k_i},\sigma_{k_{i+1}})\] makes Equation \eqref{eq:kappa:p:new} consistent with Equation \eqref{eq:new:feature:space}.

\subsubsection{\textbf{An example for  $\psi(\sigma_i, \sigma_j)=1$  and  $\theta(\sigma)=1$}}

Comparing Fig. \ref{figure:example} and Fig. \ref{figure:example:weighted}, obviously the example in Section \ref{section:motivating:example} is a special case for Theorem \ref{theorem:framework} when $\psi(\sigma_i, \sigma_j)=1$  and  $\theta(\sigma)=1$.

Now we can use Theorem \ref{theorem:framework} to calculate the consensus score for the example in Section \ref{section:motivating:example}.

\begin{example}\label{example:theorem}
	
	Consider $\rankings=\{ \ranking_1=abcdef, \ranking_2=bdcefa,  \ranking_3=bcdeghijkf, \ranking_4=badefc\}$,   based on the matrix $\mathbf{A}$ in Equation \eqref{eq:example:I}, we have 
	\[\mathbf{L}=
	\left(\begin{matrix}
	0 & 0 & 0 & 0 & 0 & 0\cr
	0 & 0 & 0 & 0 & 0 & 0\cr
	0 & 1 & 0 & 0 & 0 & 0\cr
	0 & 1 & 0 & 0 & 0 & 0\cr
	0 & 1 & 0 & 1 & 0 & 0\cr
	0 & 1 & 0 & 1 & 1 & 0\cr
	\end{matrix}
	\right), ~	
	\mathbf{L}^2=
	\left(\begin{matrix}
	0 & 0 & 0 & 0 & 0 & 0\cr
	0 & 0 & 0 & 0 & 0 & 0\cr
	0 & 0 & 0 & 0 & 0 & 0\cr
	0 & 0 & 0 & 0 & 0 & 0\cr
	0 & 1 & 0 & 0 & 0 & 0\cr
	0 & 2 & 0 & 1 & 0 & 0\cr
	\end{matrix}
	\right),	
	\]
	\[
	\mathbf{L}^3=
	\left(\begin{matrix}
	0 & 0 & 0 & 0 & 0 & 0\cr
	0 & 0 & 0 & 0 & 0 & 0\cr
	0 & 0 & 0 & 0 & 0 & 0\cr
	0 & 0 & 0 & 0 & 0 & 0\cr
	0 & 0 & 0 & 0 & 0 & 0\cr
	0 & 1 & 0 & 0 & 0 & 0\cr
	\end{matrix}
	\right),
	~ \mathbf{L}^4=
	\left(\begin{matrix}
	0 & 0 & 0 & 0 & 0 & 0\cr
	0 & 0 & 0 & 0 & 0 & 0\cr
	0 & 0 & 0 & 0 & 0 & 0\cr
	0 & 0 & 0 & 0 & 0 & 0\cr
	0 & 0 & 0 & 0 & 0 & 0\cr
	0 & 0 & 0 & 0 & 0 & 0\cr
	\end{matrix}
	\right).
	\]	
	Then, 
	\[
	\begin{array}{ll}
	\kappa_1(\rankings)=\tr(\bA)=5, & \kappa_2(\rankings)=\mathbf{z}^T\mathbf{L}\mathbf{z}=7,\\
	\kappa_3(\rankings)=\mathbf{z}^T\mathbf{L}^2\mathbf{z}=4, & \kappa_4(\rankings)=\mathbf{z}^T\mathbf{L}^3\mathbf{z}=1.
	\end{array}
	\]
\end{example}

\subsection{ $	\ell(\rankings) $ and $	\kappa(\rankings)$ } 
In Example \ref{example:theorem}, we observe that $\mathbf{L}^p = 0$ for $p\geq 4$, which implies that there are no common subsequences with length more than 4 and hence the length of the longest common subsequences of $\rankings$ is 4.  Based on this fact, the  following corollary of Theorem \ref{theorem:framework}  provides an  algorithm for calculating $\ell(\rankings)$. 
\begin{corollary}\label{lemma:lcs}
	Under the assumptions in  Theorem \ref{theorem:framework} ,   the length   $	\ell(\rankings) $ of the longest  subsequences in $\subseqs(\rankings)$ can be obtained by   
	$$
	\max \left\{  p :  \kappa_p(\rankings)>0  \right\}.
	$$
\end{corollary}

\begin{corollary}\label{lemma:rho}
	Under the assumptions in  Theorem \ref{theorem:framework} ,  
	$$
	\kappa(\rankings)=\kappa_1(\rankings)+\mathbf{z}^T (\bI-\bL)^{-1}\mathbf{z} - n
	$$
	where $n=|\ranking_x|$ and $\bI$ is the identity matrix of size $n\times n$. Consequently $\kappa(\rankings)$ can be computed in $O(n+|E|)$ time.
\end{corollary}
\begin{proof}
	Since the longest possible length of a common subsequence is $n$, Theorem~\ref{theorem:framework}  implies that
	\[
	\kappa(\rankings) = \kappa_1(\rankings) + \sum_{p=2}^n \kappa_p(\rankings)
	= \tr(\bA) + \sum_{i=1}^{n-1} \mathbf{z}^T \bL^i \mathbf{z}.
	\]
	Invoking the identity
	$(\bI - \bL)(\bI + \bL + \bL^2 +\cdots + \bL^{n-1}) = \bI - \bL^n$
	and the observation that $\bL^n= 0$ since $\bL$ is strictly lower triangular, we obtain that
	\begin{eqnarray}
	\kappa(\rankings) &=& \tr(\bA) + \mathbf{z}^T((\bI-\bL)^{-1}-\bI)\mathbf{z} \nonumber \\ 
	&=& \tr(\bA) + \mathbf{z}^T(\bI-\bL)^{-1}\mathbf{z} - n.		
	\end{eqnarray}
Now we discuss the runtime. Since $\bL$ is a strictly lower triangular matrix, $\bI-\bL$ is lower triangular. Note that computing $\bM^{-1}u$ is equivalent to solving $\bM v=u$ and, if $\bM$ is lower triangular, can be done efficiently in $O(\nnz(\bM)+n)$ time using forward elimination (degenerated Gaussian elimination), where $\nnz(\bM)$ denotes the number of non-zero entries in $\bM$. Therefore $(\bI-\bL)^{-1}\mathbf{z}$ can be computed in $O(n+|E|)$ time, and thus $\kappa(\rankings)$ in $O(n+|E|)$ time.
\end{proof}
We remark that the runtime in Corollary~\ref{lemma:rho} is significantly faster, by a factor of $n$, than the na\"ive algorithm to sum up $\kappa_p(A)$ over $p$, which would take $O(n^2+n|E|)$ time.


\subsection{Remarks} 
%

\begin{algorithm}[!t]
	\KwData{A set of rankings $\rankings=\{\ranking_1,\dots,  \ranking_N\}$}
	\KwResult{${\kappa}_1(\rankings),\dots,{\kappa}_\ell(\rankings) $}
	Pick an arbitrary\footnotemark  ~$\ranking_x = r_{x_1}\cdots r_{x_n}\in \rankings$\;
	${\mathbf{A}} \gets \left(\mathbf{0}\right)_{n\times n}$\;
			Initialize $\bM =\left( M_{ki}\right)_{N\times n}$, with  $M_{ki}=\eta_k(r_{x_i})$  \nllabel{alg:init:M}\; 
	Initialize $\mathbf{u} = (u_i)\in \mathbb{R}^n$, with  $u_{i}=\frac{1}{N}\sum_{k=1}^{N} M_{ki}$\; \nllabel{alg:init:u}
	Initialize $\mathbf{d} = (d_i)\in \mathbb{R}^n$, with  $d_{i}=\sqrt{\frac{1}{N}\sum_{k=1}^{N} |M_{ki}-\mu_i |}$\; \nllabel{alg:init:d}
	Initialize $\mathbf{W} =\left( W_{ij}\right)_{n\times n}$, with  $W_{ij}=\frac{1}{N}\sum_{k=1}^{N} |M_{ki}-M_{kj}|$\;\nllabel{alg:init:W}
	\For{$i \leftarrow 1$ \KwTo $n$}{   \nllabel{alg:loop:A:begin} 
		${A}_{ii}\gets \gamma^{d_i}\prod_{k=1}^{N} H(M_{ki})$\nllabel{alg:init:E:ii}\;
		\For{$j \leftarrow 1$ \KwTo $i-1$}
		{  
			${A}_{ij}\!\gets\!\lambda^{W_{ij}}\prod_{k=1}^{N} H(M_{ki}\!-\!M_{kj})H(\!{A}_{ii}\!)H(\!{A}_{jj}\!)$\;\nllabel{alg:init:E:ij}
		}
	}\nllabel{alg:loop:A:end} 
	${\bL}\gets\text{strictly lower triangular part of }\bA$\; 
	$p\gets 1$\;
	${\kappa}_p(\rankings)\gets \tr(\bA)$ \nllabel{alg:k:1}\;
	$\mathbf{y}\gets \mathbf{z}$\;
	\While{${\kappa}_p(\rankings)>0$}{\nllabel{alg:loop:while:begin} 
		$p\gets p+1$\;
		$\mathbf{y}\gets \mathbf{L}\mathbf{y}$ \nllabel{alg:Ly} \;
		${\kappa}_p(\rankings)\gets \mathbf{z}^T \mathbf{y}	$\;
	}
	$\ell\gets p-1$\;
	{return ${\kappa}_1(\rankings),\dots,{\kappa}_\ell(\rankings) $} 
	\caption{Pseudocode to compute $\kappa_p(\rankings)$} \label{alg}
\end{algorithm}
\footnotetext{For efficiency purpose, we can choose the ranking with the least number of items.}

Algorithm \ref{alg} is the pseudocode for Theorem \ref{theorem:framework}.  Generating $\bA$ (Line \ref{alg:loop:A:begin}--\ref{alg:loop:A:end}) takes $O(N n^2)$ time. Computing $\kappa_1(\rankings) = \tr(A)$ (Line \ref{alg:k:1}) takes $O(n)$ time. For each $p\geq 2$, the matrix-vector multiplication $\bL \mathbf{y}$ in Line \ref{alg:Ly} takes $O(n+|E|)$ time, since $\bL$ has at most $O(|E|)$ non-zero entries. In fact, $|E|=\kappa_2(\rankings)$.  Line 20 takes $O(n)$ time. Overall computing $\kappa_p(\rankings)$, after generation of $\bA$, takes $O(p(n+|E|))$ time for $p\geq 2$.

\subsubsection{Sensitivity detection by gaps and positions of items} \label{section:sensitivity}

There are some differences among $\kappa_p(\rankings)$. Note that $\kappa_1(\rankings)$ is controlled by the $\gamma$ and $d$ in Equation \eqref{equation:theta}, where $d$ is a factor to reflect the deviation of each item's positions in the rankings. Higher disagreement of  the items positions in the rankings will result in lower $\kappa_1(\rankings)$. Hence $\kappa_1(\rankings)$ is sensitivity detection of the variation of  items positions.   Whereas for $p>1$, the measure $\kappa_p(\rankings)$ takes into account the extent of the relationship between two items by incorporating $\lambda$ and $g$ in Equation \eqref{eq:psi}. 

We demonstrate such ability for sensitivity detection in terms of the gaps and positions of items in Section \ref{section:experiments}.  

\subsubsection{Duplicate rankings}
In the above example of $\rankings=\{ \ranking_1=abcdef, \ranking_2=bdcefa,  \ranking_3=bcdeghijkf, \ranking_4=badefc\}$ and its adjacent matrix $\mathbf{A}$ in Equation \eqref{eq:example:I}, we assume that there are only distinct rankings in $\rankings$. However, this is not the case, especially in the group decision making process, where there may be duplicate rankings produced by the experts. For example, we may have a multi-set $\rankings'=\{ \ranking_1=abcdef, \ranking_2=bdcefa,  \ranking_3=bcdeghijkf, \ranking_4=badefc, \ranking_5=badefc,\ranking_6=badefc$, $\ranking_7=badefc\}$, which contains duplicate ranking of $badefc$. Obviously,  $\rankings'$ has higher degree of consensus than $\rankings$. However, as $\kappa_p(\rankings')=\kappa_p(\rankings)$,  that is,  $\kappa_p(\cdot)$ cannot discriminate $\rankings'$ and $\rankings$. In order to distinguish the difference, we let $o(\ranking)$ be the number of occurrences of ranking $\ranking$ in $\rankings'$,   and  define 
\[
\hat{\kappa}(\rankings')=\kappa(\rankings')+\frac{|\rankings'|\cdot s}{|\rankings|},
\] 
where $s=\max\{  o(\ranking) : \forall \ranking\in \rankings' \}$ and $\rankings$ is the set of all distinct rankings in $\rankings'$.

{
\subsubsection{Rankings with ties}	
Rankings with ties occur when the preference scores over some items are identical. 
Let $\ranking_k=\ties_{k_1}\dots\ties_{k_n}$ be a ranking with ties, where $\ties_{k_i}$ is a set of items with an identical ranking score,  $\ties_{k_i}\cap\ties_{k_j}=\emptyset$ if $i\neq j$,  and for $i<j$, every $x\in \ties_{k_i}$ is more preferred than all $y\in \ties_{k_j}$.   If we replace Equation \eqref{eq:eta} with 
\begin{equation}\label{eq:eta:extension}
\eta_k(\sigma)=
\begin{cases}
j, & \sigma \in \ties_{k_j};\\
0, & \text{otherwise}.
\end{cases}
\end{equation}
then  Theorem \ref{theorem:framework} can be used, without any modification,  to measure consensus of rankings with ties. 

\subsubsection{Weighting scheme for top-$k$ items}
The measure $\kappa_p(\rankings)$ is a subsequence-based consensus measure and has the ability to handle top-$k$ rankings. In the case where the top-$k$ items rankings need to be weighted with greater value and the items after top-$k$  are not important  (see, e.g.~\cite{Yilmaz:2008:rank:correlation,Webber:2010,Fagin:2003}),  a slight revision to Equation \eqref{equation:theta} will  accommodate this. Consider a given cut-off value $\zeta$ for items position, we can weigh the items  with 
\[
\kappa_1(\rankings)= \sum_{\sigma\in \Sigma } H\!\left(\zeta-\mu(\sigma)-d\right) p^{\mu(\sigma)} \gamma^d 
\]
where $0<p< 1$. Clearly,  if an items is  ranked after position $\zeta$ in one of the rankings,   $ H\!\left(\zeta-\mu(\sigma)-d\right) p^{\mu(\sigma)} \gamma^d =0$.
 
}


\section{Experiments} \label{section:experiments}
This section will present how the proposed $\kappa_p(\cdot)$ and $\kappa(\cdot)$ can be applied to evaluate rank  aggregation and  compare search engine rankings. The source codes and the experimental data used are available on the github repository \footnote{\url{https://github.com/zhiweiuu/secs}}. 
\subsection{Evaluation of rank aggregation of full rankings}

\begin{table*}[h]\caption{Rankings of clustering algorithms with respect to the validation measures.}\label{table:rankings:clustering}
	\begin{center}
		{\small		\begin{tabular}{l|cccccccccc}
				\toprule
				\textbf{Validation}     & \multicolumn{10}{c}{\textbf{$\rankings_c$ -- The set of rankings for clustering algorithms}}  \\
				\textbf{measures} 	 & 1 &2&3 & 4& 5& 6& 7&8 &9 &10\\
				\toprule
				APN&          SM&FN&ST&KM&PM&HR&AG&CL&DI&MO\\
				AD&           SM&FN&KM&PM&CL&ST&DI&HR&AG&MO\\
				ADM&          FN&SM&ST&KM&CL&PM&DI&HR&AG&MO\\
				FOM&          SM&CL&KM&PM&FN&ST&DI&HR&AG&MO\\
				Connectivity& HR&AG&DI&KM&MO&SM&FN&CL&PM&ST\\
				Dunn&         HR&AG&KM&PM&DI&SM&CL&MO&FN&ST\\
				Silhouette&   HR&AG&KM&SM&CL&PM&ST&DI&FN&MO\\
				\bottomrule
			\end{tabular}  
	}	\end{center}	
	
\end{table*}
\begin{table*}[h]\caption{Rank aggregation of rankings (in Table \ref{table:rankings:clustering}) using GA and CE algorithms.}\label{table:rankings:aggregation}
	\begin{center}
		{\small		
			\begin{tabular}{l|cccccccccc}
				\toprule
				Algorithm & \multicolumn{8}{c}{ Aggregation results }\\\hline
				GA & SM & HR & KM & FN & AG & PM & CL & DI & ST & MO\\
				CE & KM & SM & PM & FN & HR & AG & CL & DI & ST & MO\\
				\bottomrule
			\end{tabular}
	}	\end{center}	
	
\end{table*}
\begin{table*}[t]
	\caption{The values of $\kappa(\rankings)$ for the rankings of clustering algorithms  when $\gamma$ and $\lambda$ vary from 1 to 0.45.}\label{table:clustering:kappa}
	{
		\begin{center}
			
			\begin{tabular}{cc|cccccccccccc}
				\toprule			 
				\multicolumn{2}{c|}{\textbf{CE}}& \multicolumn{12}{c}{$\lambda$} \\
				\multicolumn{2}{c|}{\textbf{$\kappa(\rankings_{CE})$}}&1&0.95&0.9&0.85&0.8&0.75&0.7&0.65&0.6&0.55&0.5&0.45\\
				\hline
				\multirow{12}{0.2cm}{$\gamma$}	
				&1  &19&17.589&16.376&15.336&14.448&13.692&13.05&12.508&12.05&11.666&11.344&11.076\\
				&0.95&17.945&16.534&15.321&14.282&13.394&12.637&11.996&11.453&10.996&10.611&10.29&10.021\\
				&0.9&16.964&15.553&14.34&13.301&12.413&11.657&11.015&10.472&10.015&9.631&9.309&9.04\\
				&0.85&16.055&14.644&13.431&12.391&11.503&10.747&10.105&9.563&9.105&8.721&8.399&8.13\\
				&0.8&15.213&13.803&12.589&11.55&10.662&9.906&9.264&8.721&8.264&7.88&7.558&7.289\\
				&0.75&14.438&13.027&11.814&10.775&9.886&9.13&8.489&7.946&7.489&7.104&6.782&6.514\\
				&0.7&13.726&12.315&11.102&10.063&9.174&8.418&7.777&7.234&6.777&6.392&6.07&5.802\\
				&0.65&13.075&11.664&10.451&9.411&8.523&7.767&7.125&6.583&6.125&5.741&5.419&5.151\\
				&0.6&12.482&11.071&9.858&8.818&7.93&7.174&6.532&5.99&5.532&5.148&4.826&4.557\\
				&0.55&11.944&10.533&9.32&8.281&7.392&6.636&5.995&5.452&4.995&4.61&4.288&4.02\\
				&0.5&11.46&10.049&8.836&7.796&6.908&6.152&5.51&4.967&4.51&4.126&3.804&3.535\\
				&0.45&11.026&9.615&8.402&7.362&6.474&5.718&5.076&4.533&4.076&3.692&3.37&3.101\\
				\hline		
				\\
				\toprule			 
				\multicolumn{2}{c|}{\textbf{GA}}& \multicolumn{12}{c}{$\lambda$} \\
				\multicolumn{2}{c|}{\textbf{$\kappa(\rankings_{GA})$}}&1&0.95&0.9&0.85&0.8&0.75&0.7&0.65&0.6&0.55&0.5&0.45\\
				\hline
				\multirow{12}{0.2cm}{$\gamma$}
				&1&19&17.534&16.28&15.211&14.303&13.536&12.889&12.346&11.892&11.514&11.201&10.942\\
				&0.95&17.966&16.5&15.246&14.177&13.27&12.502&11.855&11.312&10.859&10.481&10.167&9.908\\
				&0.9&17.007&15.541&14.287&13.218&12.31&11.543&10.896&10.353&9.899&9.521&9.208&8.949\\
				&0.85&16.119&14.653&13.399&12.33&11.422&10.655&10.008&9.465&9.011&8.634&8.32&8.061\\
				&0.8&15.3&13.834&12.58&11.511&10.603&9.835&9.189&8.646&8.192&7.814&7.501&7.242\\
				&0.75&14.546&13.08&11.826&10.757&9.85&9.082&8.435&7.892&7.439&7.061&6.747&6.488\\
				&0.7&13.856&12.39&11.136&10.066&9.159&8.391&7.744&7.202&6.748&6.37&6.057&5.797\\
				&0.65&13.225&11.759&10.505&9.436&8.528&7.76&7.114&6.571&6.117&5.739&5.426&5.167\\
				&0.6&12.651&11.186&9.931&8.862&7.955&7.187&6.54&5.997&5.544&5.166&4.852&4.593\\
				&0.55&12.132&10.666&9.412&8.343&7.435&6.668&6.021&5.478&5.025&4.647&4.333&4.074\\
				&0.5&11.664&10.199&8.944&7.875&6.968&6.2&5.553&5.011&4.557&4.179&3.865&3.606\\
				&0.45&11.245&9.779&8.525&7.456&6.549&5.781&5.134&4.591&4.138&3.76&3.446&3.187\\
				\hline		
			\end{tabular} 	
	\end{center}
	}
\end{table*}

Rank aggregation, with wide applications in decision making systems, machine learning and social science,  is the problem of how to combine many rankings in order to obtain one consensus ranking \cite{Brancotte:rank:agg:ties,Ivan:group:decision:2012}. Dwork et al.\ has proved that, even if $|\rankings|=4$,  obtaining an optimal aggregation with the Kendall $\tau$ index is NP-hard~\cite{Dwork:2001:rank:aggregation}.  Another problem with rank aggregation is the lack of ground truth for evaluation. Here, we show how we can use the proposed $\kappa_p(\cdot)$ to evaluate the ranking aggregation results.  

The rankings used in this experiment are the seven rankings (shown in Table \ref{table:rankings:clustering})  of 10 clustering algorithms with respect to 7 different validation measures when they are used to cluster microarray data into five clusters \cite{Pihur2009}. The details for the 7 validation measures of  APN, AD, ADM, FOM, connectivity, Dunn and Silhouette and the details for  10 clustering algorithms of  SM, FN, KM, PM, HR, AG, CL, DI and  MO can be found in \cite{Pihur2009}.  Here, we use $\rankings_c$ to denote the set of the 7 rankings. Two rank aggregation algorithms (the Cross-Entropy Monte Carlo algorithm (CE)   and the Genetic algorithm (GA)  ) have been used to aggregate the 7 rankings in  \cite{Pihur2009} and the aggregated rankings are shown in Table \ref{table:rankings:aggregation}.   We create two sets of rankings $\rankings_{GA}$ and $\rankings_{CE}$ by adding each aggregated ranking into $\rankings_c$,
\begin{gather*}
\rankings_{GA}=\rankings_c\cup \{GA\},\\
\rankings_{CE}=\rankings_c\cup \{CE\}.
\end{gather*}

Based on the property in Equation \eqref{eq:kappa:property3}, clearly $\kappa(\rankings_{GA})\leq \kappa(\rankings_c)$ and $\kappa(\rankings_{CE})\leq \kappa(\rankings_c)$. Therefore, we would reasonably expect that a better aggregated ranking would result less decline of $\kappa(\cdot)$ when adding the aggregated ranking into $\rankings_c$. Table \ref{table:clustering:kappa} shows experimental results for  $\kappa(\rankings_{GA})$ and $\kappa(\rankings_{CE})$ by varying values of $\gamma$ and $\lambda$ from 1 to 0.45. Figure \ref{fig:clustering:exp} shows the changes of $\kappa_p(\cdot)$ for both GA and CE. We find that the aggregated ranking by the CE algorithm is more sensitive to $\gamma$ while the aggregated ranking by the GA algorithm is more sensitive to $\lambda$. 
\begin{figure}[t] 
	\centering
	\subfloat[$\kappa_1(\rankings)$ with respect to $\gamma$]{
		\includegraphics[width=0.9\linewidth]{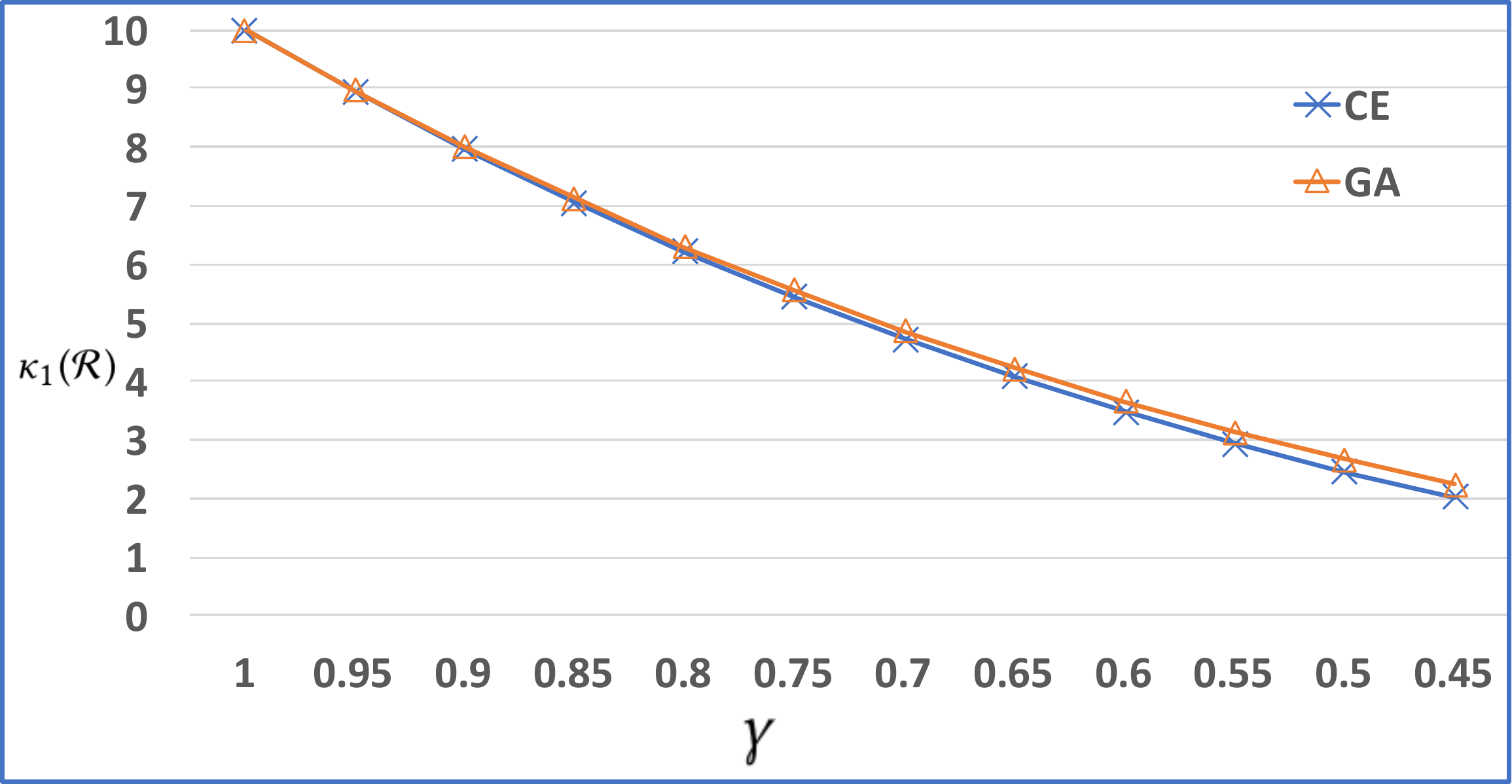}
	}	
	
	\subfloat[$\kappa_2(\rankings)$ with respect to $\gamma$]{
		\includegraphics[width=0.9\linewidth]{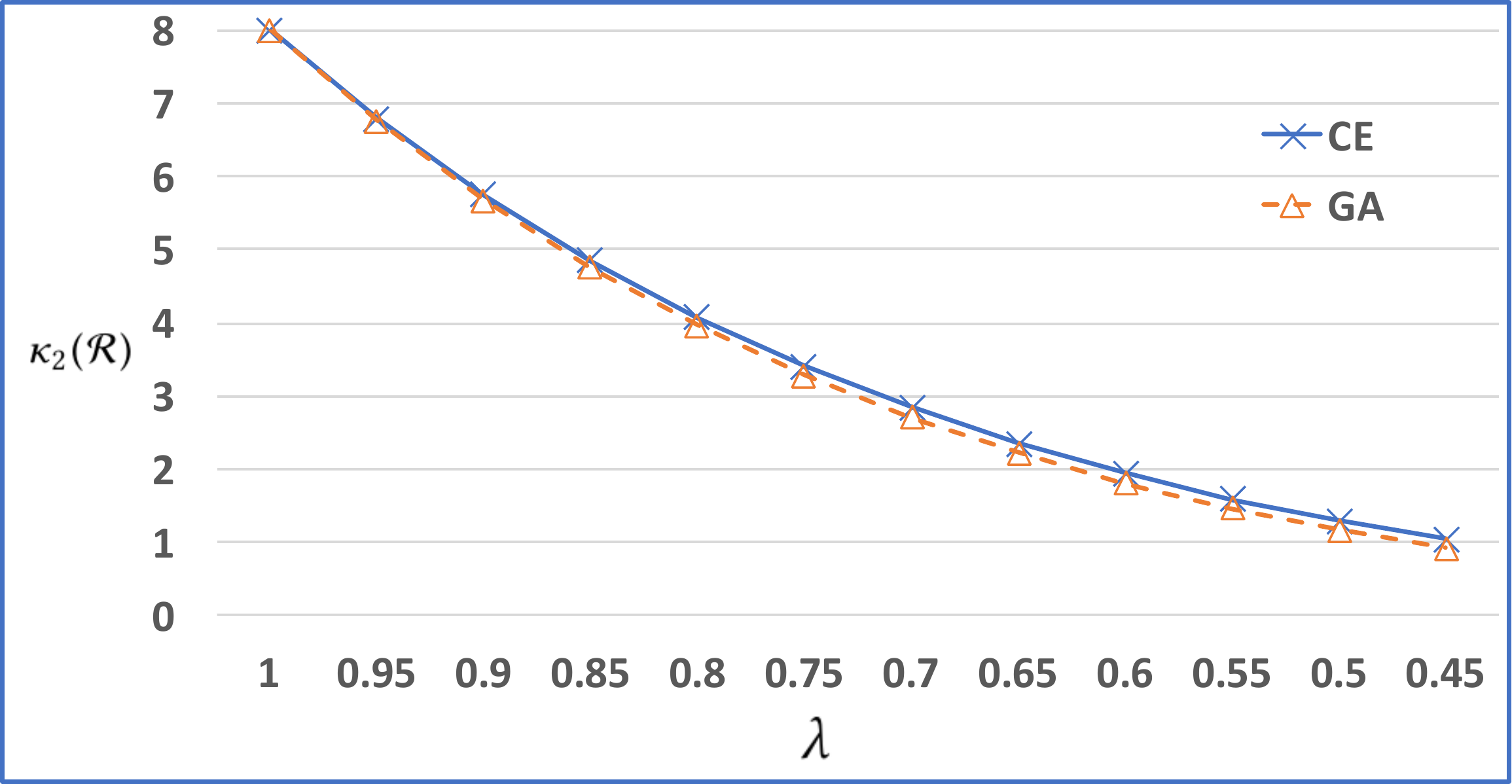}			
	}	
	
	\subfloat[$\kappa_3(\rankings)$   with respect to $\lambda$]{
		\includegraphics[width=0.9\linewidth]{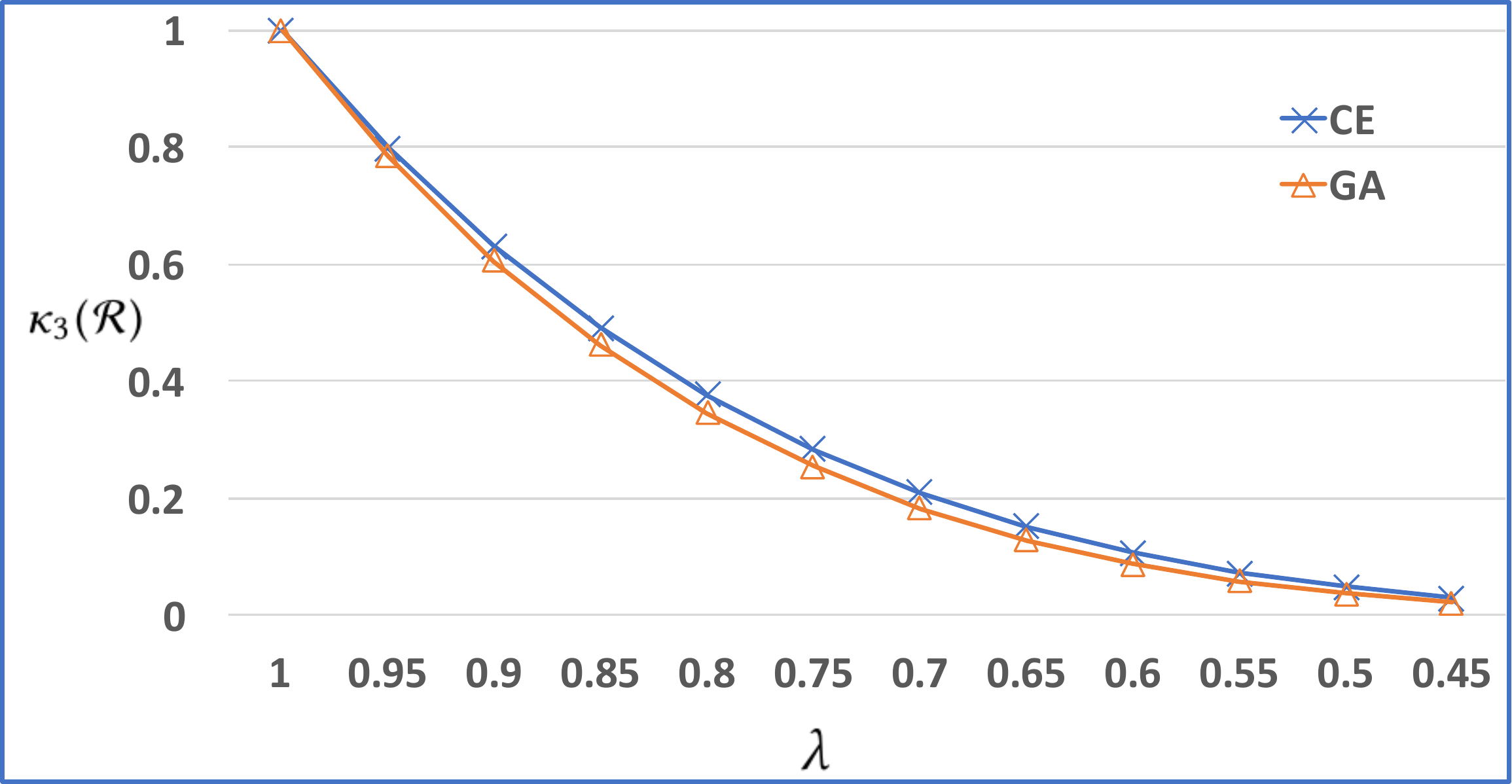}
	}
	
	\caption{The changes of $\kappa_1(\rankings)$, $\kappa_2(\rankings)$, $\kappa_3(\rankings)$  with respect to $\gamma$ and $\lambda$ for evaluating rank aggregation.}
	\label{fig:clustering:exp}
\end{figure}  

\subsection{Evaluation of consensus for top-$k$ rankings}

 \begin{table*} \caption{\small Rankings of top 25 links returned from Google and Bing with the given key words ($BF$ for ``bond films'', $BM$ for ``bond movies'', $0M$ for ``007 movies'', $0F$ for ``007 films'', $JF$ for ``james bond films'', and $JM$ for ``james bond movies''). $\google$ and $\bing$ stand for Google and Bing respectively. The numbers are the ids for the links returned from the search engines, and the mapping of the ids to the links  can be find on our github repository.} \label{table:search:rankings}
	{
	\begin{center}
		
		\begin{tabular}{ll}
			\hline
			$\google$ &  Top 25 links from Google   \\\hline
			${BF}$&\framebox[1\width]{0},68,\framebox[1\width]{9,59,11,5,3},69,79,70,21,4,36,32,76,40,\framebox[1\width]{60},51,80,81,42,29,82,83,73\\
			${BM}$&5,0,9,11,59,76,3,36,21,79,90,70,4,60,93,35,50,40,42,92,86,87,73,98,94\\
			${0M}$&0,9,11,5,3,59,70,84,85,21,76,4,32,42,51,62,12,80,67,60,55,86,87,73,29\\
			${0F}$&0,9,3,11,5,2,59,4,32,21,60,35,42,61,62,12,51,17,63,64,55,65,66,40,67\\
			${JF}$&5,0,68,9,3,11,59,69,70,71,4,21,60,32,36,61,65,72,73,58,74,75,76,77,78\\
			${JM}$&0,9,59,5,11,3,88,60,89,69,90,70,32,91,75,92,93,94,61,40,86,87,95,96,97\\ \hline\hline
			$\bing$ & Top 25 links from Bing \\\hline
			${BF}$&\framebox[1\width]{0,9,1,11},36,\framebox[1\width]{8,4,2},22,5,37,15,38,6,28,34,29,\framebox[1\width]{19},39,40,35,24,41,32,42\\
			${BM}$&0,1,9,2,11,5,50,8,4,56,57,22,3,40,51,7,41,15,19,37,36,55,42,58,38\\
			${0M}$&0,7,9,2,10,8,4,6,5,1,11,14,3,40,13,43,19,44,45,46,47,48,49,17,28\\
			${0F}$&0,1,2,3,4,5,6,7,8,9,10,11,12,13,14,15,16,17,18,19,20,21,22,23,24\\
			${JF}$&9,1,25,4,7,2,0,26,5,11,15,8,27,10,28,29,30,31,6,19,32,22,33,34,35\\
			${JM}$&0,9,3,1,2,11,50,4,8,7,6,38,40,25,10,51,15,19,52,36,53,14,54,55,28\\		\hline
		\end{tabular}
	\end{center}	 	

	} 
\end{table*}

\begin{table*}[t]
	\caption{The values of $\kappa(\rankings)$ for the search rankings from Google and Bing  when $\gamma$ and $\lambda$ vary from 1 to 0.45.}\label{table:search:kappa}
	{
	 \begin{center}
	 	
		\begin{tabular}{cc|cccccccccccc}
			\toprule			 
			\multicolumn{2}{c|}{\textbf{Google}}& \multicolumn{12}{c}{$\lambda$} \\
			\multicolumn{2}{c|}{\textbf{$\kappa(\rankings)$}}&1&0.95&0.9&0.85&0.8&0.75&0.7&0.65&0.6&0.55&0.5&0.45\\
			\hline
			\multirow{12}{0.2cm}{$\gamma$}&	1&33&24.502&19.136&15.698&13.438&11.896&10.798&9.982&9.351&8.85&8.446&8.117\\
			&0.95&32.475&23.977&18.611&15.173&12.913&11.371&10.274&9.457&8.826&8.325&7.921&7.592\\
			&0.9&31.982&23.483&18.118&14.68&12.42&10.878&9.78&8.964&8.333&7.832&7.428&7.099\\
			&0.85&31.518&23.02&17.654&14.216&11.956&10.414&9.316&8.5&7.869&7.368&6.964&6.635\\
			&0.8&31.081&22.583&17.217&13.779&11.519&9.977&8.879&8.063&7.432&6.931&6.527&6.198\\
			&0.75&30.669&22.17&16.805&13.367&11.107&9.565&8.467&7.65&7.02&6.519&6.114&5.785\\
			&0.7&30.279&21.781&16.415&12.977&10.717&9.175&8.078&7.261&6.63&6.129&5.725&5.396\\
			&0.65&29.91&21.412&16.046&12.608&10.348&8.806&7.709&6.892&6.261&5.76&5.356&5.027\\
			&0.6&29.56&21.061&15.696&12.258&9.998&8.456&7.358&6.542&5.911&5.41&5.005&4.676\\
			&0.55&29.226&20.728&15.362&11.925&9.664&8.123&7.025&6.208&5.577&5.076&4.672&4.343\\
			&0.5&28.908&20.41&15.044&11.607&9.346&7.805&6.707&5.89&5.259&4.758&4.354&4.025\\
			&0.45&28.604&20.105&14.74&11.302&9.042&7.5&6.402&5.586&4.955&4.454&4.049&3.721	\\
			\hline\\		
			\toprule			 
			\multicolumn{2}{c|}{\textbf{Bing}}& \multicolumn{12}{c}{$\lambda$} \\
			\multicolumn{2}{c|}{\textbf{$\kappa(\rankings)$}}&1&0.95&0.9&0.85&0.8&0.75&0.7&0.65&0.6&0.55&0.5&0.45\\
			\hline
			\multirow{12}{0.2cm}{$\gamma$}&	1&23&16.392&12.835&10.887&9.788&9.143&8.745&8.49&8.32&8.207&8.13&8.079\\
			&0.95&22.147&15.539&11.982&10.034&8.935&8.29&7.892&7.637&7.467&7.354&7.277&7.226\\
			&0.9&21.355&14.747&11.189&9.242&8.143&7.497&7.1&6.844&6.675&6.561&6.484&6.433\\
			&0.85&20.62&14.012&10.455&8.507&7.409&6.763&6.365&6.11&5.941&5.827&5.75&5.699\\
			&0.8&19.942&13.334&9.777&7.829&6.731&6.085&5.687&5.432&5.263&5.149&5.072&5.021\\
			&0.75&19.318&12.71&9.152&7.205&6.106&5.46&5.063&4.807&4.638&4.524&4.448&4.396\\
			&0.7&18.745&12.137&8.58&6.632&5.534&4.888&4.49&4.235&4.066&3.952&3.875&3.824\\
			&0.65&18.222&11.614&8.057&6.109&5.01&4.365&3.967&3.712&3.542&3.429&3.352&3.301\\
			&0.6&17.746&11.138&7.581&5.633&4.535&3.889&3.491&3.236&3.067&2.953&2.876&2.825\\
			&0.55&17.316&10.708&7.15&5.203&4.104&3.458&3.061&2.805&2.636&2.522&2.446&2.394\\
			&0.5&16.928&10.321&6.763&4.816&3.717&3.071&2.673&2.418&2.249&2.135&2.058&2.007\\
			&0.45&16.582&9.974&6.417&4.469&3.371&2.725&2.327&2.072&1.903&1.789&1.712&1.661\\
			\hline		
		\end{tabular} 
	 	
	 \end{center}		
	
	}
\end{table*}

In this section, we show how the proposed consensus measure can be used to evaluate top-$k$ rankings \cite{Fagin:2003}. Here, we are interested in the top-25 items from Google and Bing searches.  Twelve top-25 rankings  used in this experiment are  the search rankings from Google and Bing by using 6 ``related'' key words: ``Bond films'', ``Bond Movies'', ``007 films'', ``007 movies'', ``James Bond films'' and ``James Bond movies'', with 6 rankings from Google and 6 rankings from Bing. Table \ref{table:search:rankings} shows the twelve rankings.  As these key words refer to an identical concept from human perspective,  we want to know how close or related these rankings are, which can be evaluated by the proposed consensus measure $\kappa(\rankings)$ and $\kappa_p(\rankings)$, in terms of ``relatedness'' or ``closeness''.   More details about the  extracted rankings can be found on the github repository. 

The $\kappa(\rankings)$ values for Bing and Google rankings are shown in  Table  \ref{table:search:kappa}.  The results show that for the given key words, Google has consistently higher $\kappa(\rankings)$ values than Bing when both $\lambda$ and  $\gamma$ varies from 1 to 0.45. However, it is difficult to discern which search engine results are more sensitive to $\gamma$ and $\lambda$. Therefore, we show the changes of $\kappa_p(\rankings)$ in terms of $\gamma$ and $\lambda$ in Figure \ref{fig:search:kappa:p}. In Figure \ref{fig:search:kappa:1}, though Bing has 8 links in common ($\kappa_1(\rankings)=8$) and Google has only 7 links in common ($\kappa_1(\rankings)=7$), Bing's search results are more sensitive to the $\gamma$, which shows that the deviation of the links positions in Bing's search results are bigger than that  by Google search. Especially, when $\gamma\leq 0.85$, the $\kappa_1(\rankings)$ for Google's rankings is in fact higher than that from Bing's rankings. Also, as shown in Figure \ref{fig:search:kappa:p}, both $\kappa_2(\rankings)$ and $\kappa_3(\rankings)$ suggest that Google's rankings have a higher degree of consensus than Bing's rankings.

\begin{figure}[t]
	\begin{center}
		
		\subfloat[$\kappa_1(\rankings)$ with respect to $\gamma$]{\label{fig:search:kappa:1}
			\includegraphics[width=0.9\linewidth]{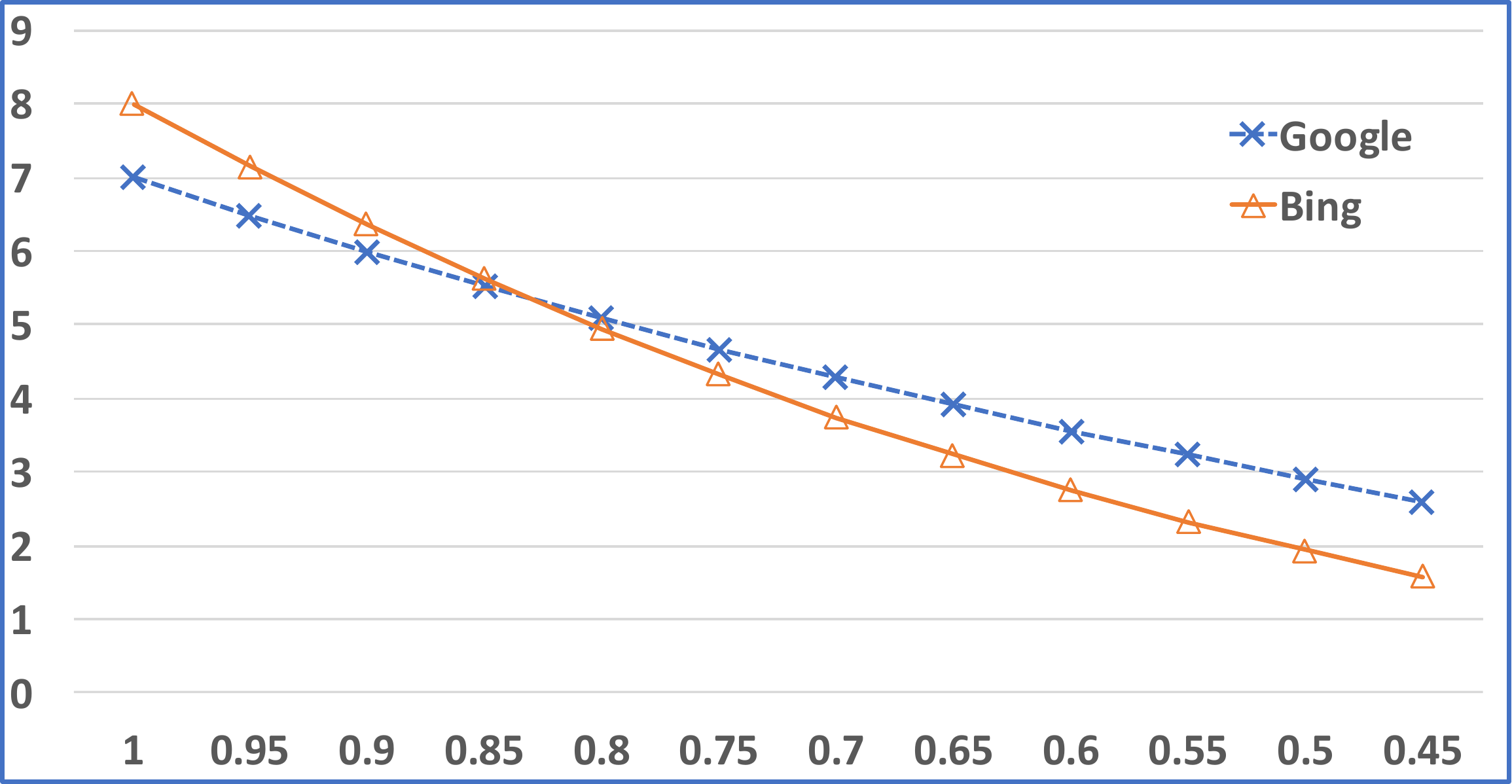}
			%
		}
		
		\subfloat[$\kappa_2(\rankings)$ with respect to $\lambda$]{
			\includegraphics[width=0.9\linewidth]{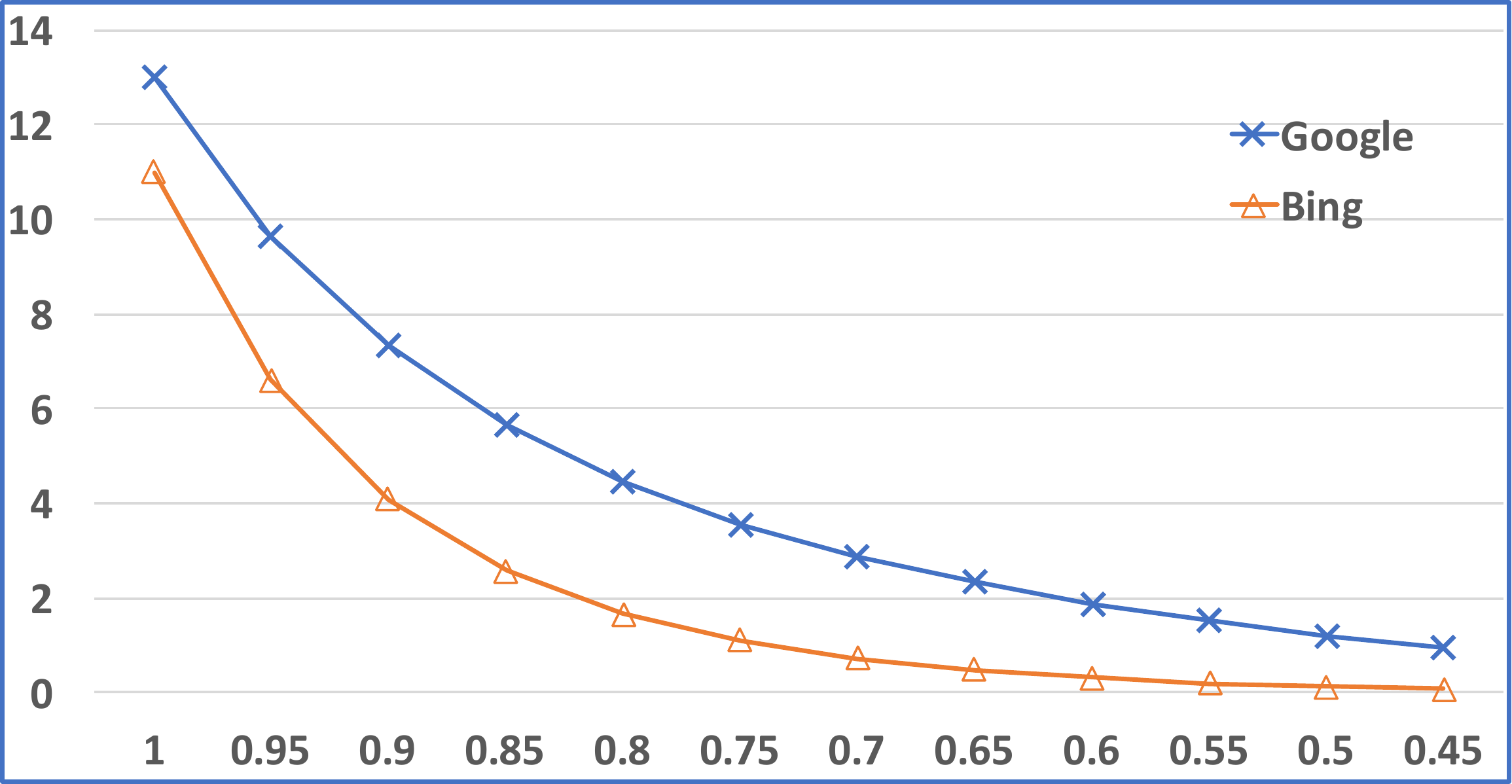}			
		}
		
		\subfloat[$\kappa_3(\rankings)$   with respect to $\lambda$]{
			\includegraphics[width=0.9\linewidth]{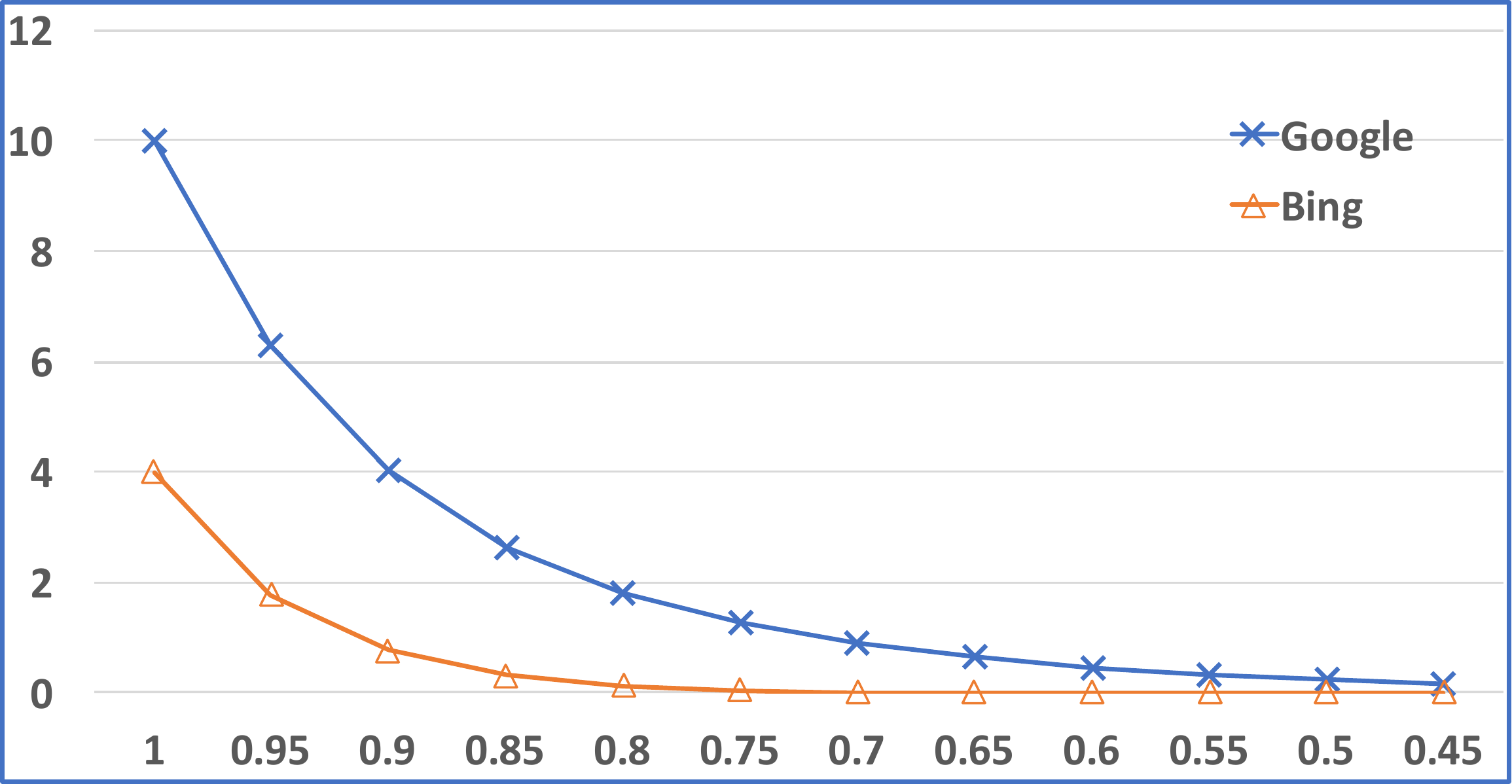}
			%
		}
	\end{center}	
	
	%
	\caption{The changes of $\kappa_1(\rankings)$, $\kappa_2(\rankings)$, $\kappa_3(\rankings)$  with respect to $\gamma$ and $\lambda$ for evaluating search rankings from Google and Bing. }
	\label{fig:search:kappa:p}
\end{figure}

\section{Conclusion and future work}\label{section:conclusion}

This paper introduces a novel approach for consensus measure of rankings by using graph representation, in which the vertices are the items and the edges are the relationship of items in the rankings.  Such representation leads to various algorithms for consensus measure in terms of different aspects in the rankings, including the number of common patterns, the number of common patterns with fixed length and the length of the longest common patterns. We present  how the proposed approaches can be used to evaluate  rank aggregation and compare search engines rankings. 

In future, we will look into the property, shown in Equation \eqref{eq:kappa:property3} and use it to define a new objective function for rank aggregation so that the proposed approaches can be used  to develop elastic algorithms for rank aggregation. A challenging task for future is to extract the common patterns of rankings and use these common pattern to define a probabilistic model for evaluating rankings generated from different systems.

\bibliography{bib}
\end{document}